\documentclass[12pt,a4paper]{article}
\pdfoutput=1
\usepackage{amsfonts}
\usepackage{amsthm}
\usepackage{amsmath}
\usepackage{mathrsfs}
\usepackage{amssymb}
\usepackage{graphicx}
\usepackage{caption}
\RequirePackage[colorlinks,linkcolor=red,citecolor=red,urlcolor=red]{hyperref}
\usepackage{latexsym,epsf,dsfont,color,eurosym,mathrsfs,url}

\newlength{\fixboxwidth}
\newtheorem{remark}{Remark}
\newtheorem{prop}{Proposition}

\newtheorem{definition}{Definition}
\newtheorem{lemma}{Lemma}
\newtheorem{theorem}{Theorem}
\newtheorem{example}{Example}
\newtheorem{assumption}{Assumption}

\newcommand{\N}{{\mathbb N}}
\newcommand{\R}{{\mathbb R}}
\newcommand{\E}{{\mathbb E}}

\newcommand{\RR}{\mathbb R}
\newcommand{\ZZ}{\mathbb Z}
\newcommand{\NN}{\mathbb N}

\numberwithin{equation}{section}

\title{\textbf{Learning rates for the risk of kernel based quantile regression estimators in additive models}$^\dag$\footnotetext{\dag
~The work by A. Christmann described in this paper is partially supported by a grant of the Deutsche
Forschungsgesellschaft [Project No. CH/291/2-1]. The work by D. X. Zhou described in this paper is supported by a grant
from the Research Grants Council of Hong Kong [Project No. CityU 104710].}}

\author{\textbf{Andreas Christmann}$^1$ and \textbf{Ding-Xuan Zhou}$^2$ \\
$^1$ University of Bayreuth, Germany\\
$^2$ City University of Hong Kong, China\\}

\date{May 11, 2014}

\begin{document}

\maketitle

\begin{abstract}
%\AC New or modified text by Andreas.\ACe  \\
%\verb! \AC New or modified text by Andreas. \ACe !
%\fix{Dear Ding-Xuan, I mostly wrote my comments in colour to make sure, that it is not necessary
%to check everything.}
%\DX New or modified text by Ding-Xuan.\DXe \\
%\verb! \DX New or modified text by Ding-Xuan. \DXe !

Additive models play an important role in semiparametric statistics.
This paper gives learning rates for regularized kernel based methods for additive models.
These learning rates compare favourably in particular in high dimensions to recent results on
optimal learning rates for purely nonparametric regularized kernel based quantile regression using the Gaussian radial basis
function kernel, provided the assumption of an additive model is valid.
Additionally, a concrete example is presented
to show that a Gaussian function depending only on one variable lies in a reproducing kernel Hilbert space generated by an additive Gaussian kernel,
but does not belong to the reproducing kernel Hilbert space generated by the multivariate Gaussian kernel of the same variance.
\end{abstract}


%\noindent{\bf AMS Subject Classification Numbers.} 68Q32, 41A25

\section{Introduction}\label{IntroductionSection}

Additive models \cite{Stone1985, HastieTibshirani1986, HastieTibshirani} provide an important
family of models for semiparametric regression or classification. 
Some reasons for the success of additive models are their increased flexibility when compared to
linear or generalized linear models and their increased interpretability when compared to fully nonparametric models. 
It is well-known that good estimators in additive models are in general less prone to the curse of high dimensionality 
than good estimators in fully nonparametric models.
Many examples of such estimators belong to the large class of regularized kernel based methods over a reproducing kernel Hilbert space $H$,
see e.g. \cite{PoggioGirosi1990,Wahba1999}.
In the last years many interesting results on learning rates of regularized kernel based models 
for additive models have been published when the focus is on sparsity
and when the classical least squares loss function is used, see e.g.
\cite{LinZhang2006}, \cite{Bach2008}, \cite{KoltchinskiiYuan2008},
\cite{MeierVandeGeerBuehlmann2009}, \cite{RaskuttiWainwrightYu2012}, \cite{SuzukiSugiyama2013}  and the references therein.
Of course, the least squares loss function is differentiable and has many nice mathematical properties, 
but it is only locally Lipschitz 
continuous and therefore regularized kernel based methods
based on this loss function typically suffer on bad statistical robustness properties, even if the 
kernel is bounded. This is in sharp contrast to
kernel methods based on a Lipschitz continuous loss function and on a bounded loss function, 
where results on upper bounds for the maxbias bias and 
on a bounded influence function are known, see e.g. \cite{ChristmannVanMessemSteinwart2009} for the general case
and \cite{ChristmannHable2012} for additive models.

Therefore, we will here consider the case of regularized kernel based methods based on a general convex and 
Lipschitz continuous loss function, on a general kernel, and on the classical regularizing term $\lambda \| \cdot \|_H^2$
for some $\lambda>0$ which is a smoothness penalty but not a sparsity penalty, 
see e.g.
\cite{Vapnik1995,Vapnik1998,SchoelkopfSmola2002,SuykensEtAl2002,CuckerZhou2007,
SteinwartChristmann2008a,HofmannSchoelkopfSmola2008,EbertsSteinwart2013}. 
Such regularized kernel based methods are now often called support vector machines (SVMs), although
the notation was historically used for such methods based on the special hinge loss function and 
for special kernels only, we refer to
\cite{VapnikLerner1963,BoserGuyonVapnik1992,CortesVapnik1995}.

In this paper we address the open question, whether an SVM with an additive kernel can provide
a substantially better learning rate in high dimensions than an SVM with a general kernel, say a classical Gaussian RBF kernel,
if the assumption of an additive model is satisfied. 
Our leading example covers learning rates for quantile regression based on the Lipschitz continuous but non-differentiable 
pinball loss function, which is also called check function in the literature, see e.g. 
\cite{KoenkerBassett1978} and \cite{Koenker2005} for parametric quantile regression and
\cite{SchoelkopfEtAl2000}, \cite{TakeuchiEtAl2006}, and \cite{SC2} for kernel based quantile regression.
We will not address the question how to check whether the assumption of an additive model is satisfied
because this would be a topic of a paper of its own. 
Of course, a practical approach might be to fit both models and compare their risks evaluated for test data.
For the same reason we will also not cover sparsity.

Consistency of support vector machines generated by additive kernels for additive models was considered in \cite{ChristmannHable2012}.
In this paper we establish learning rates for these algorithms.
Let us recall the framework with a complete separable metric space ${\mathcal X}$ as the input space and a closed
subset ${\mathcal Y}$ of $\RR$ as the output space. A Borel probability measure $P$ on ${\mathcal Z}:= {\mathcal X} \times {\mathcal Y}$
is used to model the learning problem and an independent and identically distributed sample $D_n =\{(x_i, y_i)\}_{i=1}^n$
is drawn according to $P$ for learning. A loss function $L: {\mathcal X} \times {\mathcal Y} \times \RR \to [0, \infty)$
is used to measure the quality of a prediction function $f: {\mathcal X} \to \RR$ by the local error $L(x, y, f(x))$.
\emph{Throughout the paper we assume that $L$ is measurable, $L(x, y, y)=0$, convex with respect to the third variable, and
uniformly Lipschitz continuous satisfying
\begin{equation}\label{Lipsch}
\sup_{(x, y) \in {\mathcal Z}} |L(x, y, t) - L(x, y, t')| \leq |L|_1 |t-t'| \qquad \forall t, t' \in \RR
\end{equation}
with a finite constant $|L|_1 \in (0, \infty)$.}

Support vector machines (SVMs) considered here are kernel-based re\-gularization schemes in a reproducing kernel Hilbert space (RKHS) $H$ generated by a Mercer kernel $k: {\mathcal X} \times {\mathcal X} \to \RR$.
With a shifted loss function $L^*: {\mathcal X} \times {\mathcal Y} \times \RR \to \RR$ introduced for dealing even with heavy-tailed distributions
as
$L^* (x, y, t) = L(x, y, t) - L(x, y, 0)$, they take the form $f_{L, {\mathbb D}_n, \lambda}$ where for a general Borel measure $\rho$ on ${\mathcal Z}$, the function $f_{L, \rho, \lambda}$ is defined by
\begin{equation}\label{algor}
f_{L, \rho, \lambda} = \arg \min_{f\in H} \left\{{\mathcal R}_{L^*, \rho} (f) + \lambda \|f\|_H^2\right\}, \quad {\mathcal R}_{L^*, \rho} (f)=\int_{{\mathcal Z}} L^* (x, y, f(x)) \,d \rho (x, y),
\end{equation}
where $\lambda >0$ is a regularization parameter.
The idea to shift a loss function has a long history, see e.g. \cite{Huber1967} in the context of M-estimators.
It was shown in \cite{ChristmannVanMessemSteinwart2009} that $f_{L, \rho, \lambda}$ is also a minimizer of the following optimization problem
involving the original loss function $L$ if a minimizer exists:
\begin{equation}\label{algorOld}
\min_{f\in H} \left\{\int_{{\mathcal Z}} L (x, y, f(x)) \,d \rho (x, y) + \lambda \|f\|_H^2\right\}.
\end{equation}

The additive model we consider consists of the \emph{input space decomposition}
${\mathcal X} = {\mathcal X}_1 \times \ldots \times {\mathcal X}_s$ with each ${\mathcal X}_j$ a complete separable metric space
and a \emph{hypothesis space}
\begin{equation}\label{additive}
{\mathcal F} = \left\{f_1 + \ldots + f_s: f_j \in {\mathcal F}_j, j=1, \ldots, s\right\},
\end{equation}
where ${\mathcal F}_j$ is a set of functions $f_j: {\mathcal X}_j \to \RR$ each of which is also
identified as a map $(x_1, \ldots, x_s) \longmapsto f_j (x_j)$ from ${\mathcal X}$ to $\RR$.
Hence the functions from ${\mathcal F}$ take the additive form $f(x_1, \ldots, x_s) = f_1 (x_1) + \ldots + f_s (x_s)$.
We mention, that there is strictly speaking a notational problem here, because in the previous formula
each quantity $x_j$ is an element of the set $\mathcal{X}_j$ which is a subset of the full input space $\mathcal{X}$, $j=1,\ldots,s$, whereas in the definition of
sample $D_n =\{(x_i, y_i)\}_{i=1}^n$ each quantity $x_i$ is an element of the full input space $\mathcal{X}$, where $i=1,\ldots,n$.
Because these notations will only be used in different places and because we do not expect any misunderstandings, we think this notation
is easier and more intuitive than specifying these quantities with different symbols.

The additive kernel $k= k_1 + \ldots + k_s$ is defined in terms of Mercer kernels $k_j$ on ${\mathcal X}_j$ as
$$ k\left((x_1, \ldots, x_s), (x'_1, \ldots, x'_s)\right) = k_1 (x_1, x'_1) + \ldots + k_s (x_s, x'_s). $$
It generates an RKHS $H$ which can be written in terms of the RKHS $H_j$ generated by $k_j$ on ${\mathcal X}_j$ corresponding to the form (\ref{additive}) as
$$ H = \left\{f_1 + \ldots + f_s: f_j \in H_j, j=1, \ldots, s\right\} $$
with norm given by
  \begin{eqnarray*}
    \|f\|_H^2
    &=&     \min_{f=f_1+\ldots+f_s \atop
          f_1\in H_1,\ldots, f_s\in H_s }\!
    \|f_1\|_{H_1}^2+\ldots+\|f_s\|_{H_s}^2
    \,\,.
  \end{eqnarray*}
The norm of $f:=f_1+\ldots+f_s$ satisfies
\begin{equation} \label{ACNormInequality}
\|f_1 + \ldots + f_s\|^2_H \leq \|f_1\|^2_{H_1} + \ldots + \|f_s\|^2_{H_s}, \qquad f_1 \in H_1, \ldots, f_s \in H_s.
\end{equation}

To illustrate advantages of additive models, we
provide two examples of comparing additive with product kernels.
The first example deals with Gaussian RBF kernels.
All proofs will be given in Section \ref{ProofSection}.

\begin{example}\label{GaussAdd}
Let $s=2$, ${\mathcal X}_1 ={\mathcal X}_2 =[0, 1]$ and ${\mathcal
X} = [0, 1]^2.$ Let $\sigma >0$ and
$$ k_1 (u, v) = k_2 (u, v) =
\exp\left(-\frac{|u-v|^2}{\sigma^2}\right), \qquad u, v\in [0,
1]. $$ The additive kernel $k\left((x_1, x_2), (x'_1, x'_2)\right)
= k_1 (x_1, x'_1) + k_2 (x_2, x'_2)$ is given by
\begin{equation}\label{GaussAddForm}
k\left((x_1, x_2), (x'_1, x'_2)\right) =\exp\left(-\frac{|x_1 -
x'_1|^2}{\sigma^2}\right) +\exp\left(-\frac{|x_2 -
x'_2|^2}{\sigma^2}\right).
\end{equation}
Furthermore, the product kernel $k^{\Pi} \left((x_1, x_2), (x'_1,
x'_2)\right) = k_1 (x_1, x'_1) \cdot k_2 (x_2, x'_2)$ is the
standard Gaussian kernel given by
\begin{eqnarray}\label{GaussProdForm}
k^{\Pi} \left((x_1, x_2), (x'_1, x'_2)\right)
& = & \exp\left(-\frac{|x_1 - x'_1|^2 + |x_2 - x'_2|^2}{\sigma^2}\right) \\
& = & \exp\left(-\frac{\left|(x_1, x_2) -
(x'_1, x'_2)\right|^2}{\sigma^2}\right).
\end{eqnarray}
Define a Gaussian function $f$ on ${\mathcal X} = [0, 1]^2$ depending only on one variable by
\begin{equation}\label{gaussfcn}
f\left(x_1, x_2\right)
=\exp\left(-\frac{|x_1|^2}{\sigma^2}\right).
\end{equation}
Then $f\in H$ but
\begin{equation}\label{notinclude}
f \not\in H_{k^{\Pi}},
\end{equation}
where $H_{k^{\Pi}}$ denotes the RKHS generated by the standard
Gaussian RBF kernel $k^{\Pi}$.
\end{example}

The second example is about Sobolev kernels.

\begin{example}\label{SobolvAdd}
Let $2 \leq s\in \NN$, ${\mathcal X}_1 = \ldots ={\mathcal X}_s
=[0, 1]$ and ${\mathcal X} = [0, 1]^s.$ Let
$$
W^1[0, 1]
:=
\bigl\{u\in L_2([0,1]); D^\alpha u \in L_2([0,1]) \mbox{~for~all~}|\alpha|\le 1\bigr\}
$$
be the
Sobolev space consisting of all square integrable univariate functions whose
derivative is also square integrable. It is an RKHS with a Mercer
kernel $k^*$ defined on $[0, 1]^2$. If we take all the Mercer
kernels $k_1, \ldots, k_s$ to be $k^*$, then $H_j =W^1 [0, 1]$ for
each $j$. The additive kernel $k$ is also a Mercer kernel and
defines an RKHS
$$
H= H_1 + \ldots + H_s =\left\{f_1 (x_1) + \ldots
+ f_s (x_s): f_1, \ldots, f_s \in W^1 [0, 1]\right\}.
$$
However,
the multivariate Sobolev space $W^1 ([0, 1]^s)$, consisting of all
square integrable functions whose partial derivatives are all
square integrable, contains discontinuous functions and is not an
RKHS.
\end{example}

Denote the marginal distribution of $P$ on ${\mathcal X}_j$ as $P_{{\mathcal X}_j}$. Under the assumption that $H_j \subset {\mathcal F}_j \subset L_1 (P_{{\mathcal X}_j})$ for each $j$ and that $H_j$ is dense in ${\mathcal F}_j$ in the $L_1 (P_{{\mathcal X}_j})$-metric, it was proved in
\cite{ChristmannHable2012} that
$$ {\mathcal R}_{L^*, P} (f_{L, {\mathbb D}_n, \lambda}) \to {\mathcal R}^*_{L^*, P, {\mathcal F}}:= \inf_{f\in {\mathcal F}}{\mathcal R}_{L^*, P}(f) \qquad (n\to\infty) $$
in probability as long as $\lambda = \lambda_n$ satisfies $\lim_{n\to\infty} \lambda_n =0$ and $\lim_{n\to\infty} \lambda^2_n n =\infty$.

The rest of the paper has the following structure.
Section \ref{RatesSection} contains our main results on learning rates for SVMs based on additive kernels. Learning rates for quantile regression
are treated as important special cases.
Section \ref{ComparisonSection} contains a comparison of our results with other learning rates published recently.
Section \ref{ProofSection} contains all the proofs and some results which can be interesting in their own.

%--------------------------------------------------------------------------------------------------------------
\section{Main results on learning rates}\label{RatesSection}

In this paper we provide some learning rates for the support vector machines generated by additive kernels for additive models which helps improve the quantitative understanding presented in
\cite{ChristmannHable2012}. The rates are about asymptotic behaviors of the excess risk ${\mathcal R}_{L^*, P} (f_{L, {\mathbb D}_n, \lambda}) - {\mathcal
R}^*_{L^*, P, {\mathcal F}}$ and take the form $O(m^{- \alpha})$ with $\alpha>0$. They will be stated under three kinds of conditions
involving the hypothesis space $H$,
the measure $P$, the loss $L$, and the choice of the regularization parameter $\lambda$.

\subsection{Approximation error in the additive model}

The first condition is about the approximation ability of the hypothesis space $H$. Since the output function $f_{L, {\mathbb D}_n, \lambda}$ is from the hypothesis space, the learning rates of the learning algorithm depend on the approximation ability of the hypothesis space $H$ with respect to the optimal risk ${\mathcal R}^*_{L^*, P, {\mathcal F}}$ measured by the following approximation error.
\begin{definition}\label{Defapprox}
The approximation error of the triple $(H, P, \lambda)$ is defined as
\begin{equation}\label{approxerrorDef}
{\mathcal D}(\lambda) = \inf_{f\in H} \left\{{\mathcal R}_{L^*, P} (f) - {\mathcal R}^*_{L^*, P, {\mathcal F}} + \lambda \|f\|_H^2\right\}, \qquad \lambda >0.
\end{equation}
\end{definition}

To estimate the approximation error, we make an assumption about the minimizer of the risk
\begin{equation}\label{targetinF}
f^*_{{\mathcal F}, P} = \arg \inf_{f\in {\mathcal F}}{\mathcal R}_{L^*, P}(f).
\end{equation}

For each $j\in\{1,\ldots,s\}$, define the integral operator
$L_{k_j}: L_2 (P_{{\mathcal X}_j}) \to L_2 (P_{{\mathcal X}_j})$
associated with the kernel $k_j$ by
$$
L_{k_j}(f) (x_j)
=
\int_{{\mathcal X}_j} k_j (x_j, u_j) f(u_j) d P_{{\mathcal X}_j} (u_j),
\qquad x_j \in {\mathcal X}_j, f\in L_2 (P_{{\mathcal X}_j}).
$$
We mention that $L_{k_j}$ is a compact and positive operator on $L_2 (P_{{\mathcal X}_j})$.
Hence we can find its normalized eigenpairs $((\lambda_{j,\ell}, \psi_{j,\ell}))_{\ell\in\N}$ such that
$(\psi_{j,\ell})_{\ell\in\N}$ is an orthonormal basis of $L_2 (P_{{\mathcal X}_j})$ and
$\lambda_{j,\ell} \to 0$ as $\ell\to\infty$.
Fix $r>0$. Then we can define the $r$-th power $L_{k_j}^r$ of $L_{k_j}$ by
$$
  L_{k_j}^r \Bigl(\sum_\ell c_{j,\ell} \, \psi_{j,\ell}\Bigr)
  =
  \sum_\ell c_{j,\ell} \lambda_{j,\ell}^r \, \psi_{j,\ell},
  \quad \forall (c_{j,\ell})_{\ell\in\N} \in \ell_2.
$$
This is a positive and bounded operator and its range is well-defined.
The assumption $f_j^*=L_{k_j}^{r}(g_j^*)$ means $f_j^*$ lies in this range.

%%% \noindent{\bf Assumption 1.}
\begin{assumption}\label{assumption1}
We assume $f^*_{{\mathcal F}, P} \in L_\infty (P_{{\mathcal X}})$ and
$f^*_{{\mathcal F}, P} = f^*_1 +\ldots + f^*_s$ where for some $0< r \leq \frac{1}{2}$
and each $j\in \{1, \ldots, s\}$, $f^*_j: {\mathcal X}_j \to \RR$ is a
function of the form $f^*_j = L_{k_j}^r (g^*_j)$ with some $g^*_j
\in L_2 (P_{{\mathcal X}_j})$.
\end{assumption}

The case $r=\frac{1}{2}$ of Assumption \ref{assumption1} means each $f^*_j$ lies
in the RKHS $H_j$.

A standard condition in the literature (e.g., \cite{SZII}) for
achieving decays of the form ${\mathcal D}(\lambda)
=O(\lambda^{r})$ for the approximation error
(\ref{approxerrorDef}) is $f^*_{{\mathcal F}, P} =L_{k}^r (g^*)$
with some $g^* \in L_2 (P_{{\mathcal X}})$. Here the operator
$L_k$ is defined by
\begin{equation}
L_{k}(f) (x_1, \ldots, x_s)
=
\int_{{\mathcal X}} \Bigl(\sum_{j=1}^s k_j (x_j, x'_j))\Bigr) f(x'_1, \ldots, x'_s) d P_{{\mathcal X}} (x'_1, \ldots, x'_s).
\end{equation}
In general, this cannot be written in an additive form. However, the
hypothesis space (\ref{additive}) takes an additive form
${\mathcal F} = {\mathcal F}_1 + \ldots +{\mathcal F}_s$. So it is
natural for us to impose an additive expression $f^*_{{\mathcal
F}, P} = f^*_1 + \ldots + f^*_s$ for the target function
$f^*_{{\mathcal F}, P}$ with the component functions $f^*_j$
satisfying the power condition $f^*_j = L_{k_j}^r (g^*_j)$.

The above natural assumption leads to a technical difficulty in
estimating the approximation error: the function $f^*_j$ has no
direct connection to the marginal distribution $P_{{\mathcal
X}_j}$ projected onto ${\mathcal X}_j$, hence existing methods in
the literature (e.g., \cite{SZII}) cannot be applied directly.
Note that on the product space ${\mathcal X}_j \times {\mathcal
Y}$, there is no natural probability measure projected from $P$,
and the risk on ${\mathcal X}_j \times {\mathcal Y}$ is not
defined.

Our idea to overcome the difficulty is to introduce an
intermediate function $f_{j, \lambda}$. It may not minimize a risk
(which is not even defined). However, it approximates the
component function $f^*_j$ well. When we add up such functions
$f_{1, \lambda} + \ldots + f_{s, \lambda} \in H$, we get a good
approximation of the target function $f^*_{{\mathcal F}, P}$, and
thereby a good estimate of the approximation error. This is the
first novelty of the paper.

\begin{theorem}\label{approxerrorThm}
Under Assumption \ref{assumption1}, we have
\begin{equation}\label{approxerrorB}
{\mathcal D}(\lambda) \leq C_r \lambda^{r} \qquad \forall~ 0< \lambda \leq 1,
\end{equation}
where $C_r$ is the constant given by
$$ C_r = \sum_{j=1}^s \left(|L|_1 \|g^*_j\|_{L_2 (P_{{\mathcal X}_j})} + \|g^*_j\|^2_{L_2 (P_{{\mathcal X}_j})}\right). $$
\end{theorem}

\subsection{Special bounds for covering numbers in the additive model}

The second condition for our learning rates is about the capacity of the hypothesis space measured by $\ell_2$-empirical covering numbers.

\begin{definition} Let ${\cal G}$ be a set of functions on
${\mathcal Z}$ and ${\bf z}=\{z_1, \cdots, z_m\}\subset {\mathcal Z}.$
For every $\epsilon>0,$ the \textbf{covering number of ${\cal G}$} with respect to the empirical metric
$d_{2, {\bf z}}$, given by $d_{2, {\bf z}}(f,g)=\big\{\frac{1}{m}\sum_{i=1}^m\big(f(z_i)-g(z_i)\big)^2\big\}^{1/2}$ is defined as $${\mathcal N}_{2, {\bf z}}({\cal G}, \epsilon)
=\inf\Big\{\ell\in\NN : \exists \{f_i\}_{i=1}^\ell\subset {\cal G} \
\hbox{such that} \ {\cal G}=\bigcup_{i=1}^\ell \{f\in{\cal G}:
d_{2, {\bf z}}(f, f_i) \le \epsilon\}\Big\}$$
and the \textbf{$\ell_2$-empirical covering number} of ${\mathcal G}$ is defined
as
$${\cal N}({\mathcal G}, \epsilon) =\sup_{m\in\NN} \sup_{{\bf z} \in {\mathcal Z}^m} {\mathcal
N}_{2, {\bf z}}\Bigl({\mathcal G}, \epsilon\Bigr). $$
\end{definition}

%%% \noindent{\bf Assumption $2$.}
\begin{assumption}\label{assumption2}
We assume $\kappa := \sum_{j=1}^s \sup_{x_j \in {\mathcal X}_j} \sqrt{k_j (x_j, x_j)} <\infty$ and that for some $\zeta \in (0, 2)$, $c_\zeta >0$ and every $j\in\{1, \ldots, s\}$, the $\ell_2$-empirical covering number of the unit ball of $H_j$ satisfies
\begin{equation}\label{capacityB}
\log {\cal N}\left(\{f\in H_j: \|f\|_{H_j} \leq 1\}, \epsilon\right) \leq c_\zeta \left(\frac{1}{\epsilon}\right)^\zeta, \qquad \forall~ \epsilon >0.
\end{equation}
\end{assumption}

The second novelty of this paper is to observe that the additive nature of the hypothesis space yields the following nice bound
with a dimension-independent power exponent for the covering numbers of the balls of the hypothesis space $H$, to be proved in Section \ref{SampleSection}.

\begin{theorem}\label{capacityThm}
Under Assumption \ref{assumption2}, for any $R \geq 1$ and $\epsilon >0$, we have
\begin{equation}\label{capacityBH}
\log {\cal N}\left(\{f\in H: \|f\|_{H} \leq R\}, \epsilon\right) \leq s^{1+\zeta} c_\zeta \left(\frac{R}{\epsilon}\right)^\zeta, \qquad \forall \epsilon >0.
\end{equation}
\end{theorem}

\begin{remark}
The bound for the covering numbers stated in Theorem \ref{capacityThm} is special:
the power $\zeta$ is independent of the number $s$ of the components in the
additive model. It is well-known \cite{EdmundsTriebel} in the literature of function spaces that the
covering numbers of balls of the Sobolev space $W^h$ on the cube $[-1, 1]^s$ of the Euclidean space $\RR^s$ with regularity index $h > s/2$ has the following asymptotic behavior with $0< c_{h, s} < C_{h, s} < \infty$:
$$ c_{h, s} \left(\frac{R}{\epsilon}\right)^{s/h} \leq  \log {\cal N}\left(\{f\in W^h: \|f\|_{W^h} \leq R\}, \epsilon\right) \leq C_{h, s} \left(\frac{R}{\epsilon}\right)^{s/h}. $$
Here the power $\frac{s}{h}$ depends linearly on the dimension $s$.
Similar dimension-dependent bounds for the covering numbers of the RKHSs associated
with Gaussian RBF-kernels can be found in \cite{Zhoucap}. The special bound in Theorem \ref{capacityThm} demonstrates an advantage of the additive model in terms of capacity of the additive hypothesis space.
\end{remark}

\subsection{Learning rates for quantile regression}

The third condition for our learning rates is about the noise level in the measure $P$ with respect to the hypothesis space. Before stating the general condition, we consider a special case for quantile regression, to illustrate our general results. Let $0<\tau<1$ be a quantile
parameter. The quantile regression function $f_{P, \tau}$ is defined by its value $f_{P, \tau}(x)$ to
be a $\tau$-quantile of $P(\cdot|x)$, i.e., a value $u\in {\mathcal Y} =\RR$ satisfying
\begin{equation}\label{quantile}
\rho \left(\{y\in {\mathcal Y}: y\leq u\}|x\right)\geq \tau \quad \hbox{and} \quad
 \rho \left(\{y\in {\mathcal Y}: y\geq u\}|x\right)\geq 1-\tau.
\end{equation}
The regularization scheme for quantile regression considered here takes the form (\ref{algor}) with the loss function $L$ given by the pinball loss as
\begin{equation}\label{pinloss}
L(x, y, t)=
\left\{\begin{array}{ll} (1-\tau) (t-y) ,  &\hbox{ if} \ t> y, \\
-\tau (t-y), &\hbox{ if}\ t\leq y.
\end{array}\right.
\end{equation}

A noise condition on $P$ for quantile regression is defined in \cite{SC1, SC2} as follows.
To this end, let $Q$ be a probability measure on $\R$ and $\tau\in(0,1)$. Then a real number
$q_\tau$ is called $\tau$-quantile of $Q$, if and only if $q_\tau$ belongs to the set
$$
F_\tau^\ast(Q):= \bigl\{ t\in\R, Q\bigl((-\infty,t]\bigr) \ge \tau
   \mbox{~~and~~} Q\bigl([t, \infty)\bigr) \ge 1-\tau\bigr\}\,.
$$
It is well-known that $F_\tau^\ast(Q)$ is a compact interval.

%\begin{definition}\label{orignoisecond}
%Let $p\in (0, \infty]$. We say that $P$ has a $\tau$-quantile of $p$-average type $2$ if for
%every $x\in {\mathcal X},$ there exist a $\tau$-quantile $t^{\ast}\in \RR$ and constants $a_x \in (0, 2],$ $b_{x}>0$ such that for
%each $u\in [0, a_x]$,
%\begin{equation}\label{orignoise}
%\rho(\{y\in (t^{\ast}-u, t^{\ast})\}|x)\geq b_{x} u \qquad  {\rm and } \qquad  \rho(\{y\in(t^{\ast},
%t^{\ast}+u)\}|x)\geq b_x u,
%\end{equation}
%and that the function on ${\mathcal X}$ taking value $\frac{1}{b_{x} a_x}$ at $x\in {\mathcal X}$ lies in $L^{p}_{\rho_{\mathcal X}}.$
%\end{definition}

\begin{definition}\label{noisecond}
Let $\tau\in(0,1)$.
\begin{enumerate}
\item[(1)] A probability measure $Q$ on $\R$ is said to have a \textbf{$\tau$-quantile of type $2$}, if
there exist a $\tau$-quantile $t^\ast\in\R$ and a constant $b_Q>0$ such that, for all $s\in[0, 2]$,
we have
\begin{equation} \label{TauQuantileOfType2Formula}
Q\bigl((t^\ast - s, t^\ast )\bigr) \ge b_Q s
\mbox{~~and~~}
Q\bigl((t^\ast, t^\ast +s ) \bigr) \ge b_Q s\,.
\end{equation}
\item[(2)] Let $p\in (0, \infty]$. We say that a probability measure
$\rho$ on $\mathcal{X}\times\mathcal{Y}$
has a \textbf{$\tau$-quantile of $p$-average type $2$} if the conditional probability measure
$Q_x:=\rho(\cdot|x)$ has $\rho_{\mathcal{X}}$-almost surely a
$\tau$-quantile of type $2$ and the function
$$
\gamma:\mathcal{X} \to (0,\infty),
\quad
\gamma(x):= \gamma_{\rho(\cdot|x)}:= b_{\rho(\cdot|x)}\,,
$$
where $b_{\rho(\cdot|x)}>0$ is the constant defined in part (1), satisfies
$\gamma^{-1} \in L_{\rho_{\mathcal{X}}}^p$.
\end{enumerate}
\end{definition}
One can show that a distribution $Q$ having a  $\tau$-quantile of type $2$
has a unique $\tau$-quantile $t^*$.
Moreover, if  $Q$ has a Lebesgue density $h_Q$ then $Q$
has a $\tau$-quantile of type $2$ if $h_Q$
is bounded away from zero on $ [t^*-a,t^*+a]$ since we
can use $b_Q := \inf\{ h_Q(t): t\in [t^*-a,t^*+a]\}$ in (\ref{TauQuantileOfType2Formula}).
This  assumption is general enough to cover many distributions used in
parametric statistics such as
Gaussian, Student's $t$, and logistic distributions (with $Y=\R$),
Gamma and log-normal distributions (with $Y=[0,\infty)$),
and uniform and Beta distributions (with $Y=[0,1]$).

The following theorem, to be proved in Section \ref{ProofSection}, gives a learning rate for the regularization scheme (\ref{algor}) in the special case of quantile regression.

\begin{theorem}\label{quantileThm}
Suppose that $|y| \leq |L|_0$ almost surely for some constant
$|L|_0 >0$, and that each kernel $k_j$ is $C^\infty$ with
${\mathcal X}_j \subset \RR^{d_j}$ for some $d_j \in\NN$. If
Assumption \ref{assumption1} holds with $r=\frac{1}{2}$ and $P$ has a
$\tau$-quantile of $p$-average type $2$ for some $p\in (0,
\infty]$, then by taking $\lambda =n^{-\frac{4(p+1)}{3(p+2)}}$,
for any $\epsilon
>0$ and $0<\delta <1$, with confidence at least $1- \delta$ we have
\begin{eqnarray}\label{quantilerates}
\!\!\!\!\!\! {\mathcal R}_{L^*, P} (f_{L, {\mathbb D}_n, \lambda}) - {\mathcal R}^*_{L^*, P, {\mathcal F}}
&  \le  & \widetilde{C} \left(\log \frac{2}{\delta} + \log \Bigl(\log \frac{1}{\epsilon} +2\Bigr)
\right)^2 n^{\epsilon- \alpha(p) },
\end{eqnarray}
where $\widetilde{C}$ is a constant independent of $n$ and $\delta$ and
\begin{eqnarray}\label{quantilerates2}
\alpha(p)= \frac{2(p+1)}{3(p+2)}\,.
\end{eqnarray}
\end{theorem}

Please note that the exponent $\alpha(p)$ given by {(\ref{quantilerates2})}
for the learning rate in {(\ref{quantilerates})}
is independent of the quantile level $\tau$,
of the number $s$ of additive components in $f_{L^*,\mathcal{F},P}^*=f_1^*+\ldots+f_s^*$,
and of the dimensions $d_1,\ldots,d_s$ and
$$
d=\sum_{j=1}^s d_j \,.
$$
Further note that $\alpha(p)\in[\frac{1}{2}, \,\frac{2}{3})$, if $p\ge 2$, and $\alpha(p) \to \frac{2}{3}$ if $p\to\infty$.
Because $\epsilon>0$ can be arbitrarily close to $0$, the learning rate, which is independent
of the dimension $d$ and given by Theorem \ref{quantileThm}, is close to $n^{-2/3}$ for large values of $p$ and is close to $n^{-1/2}$ or better, if $p\ge 2$.

\subsection{General learning rates}

To state our general learning rates, we need an assumption
on a \emph{variance-expectation bound} which is similar to Definition \ref{noisecond} in the special case of quantile regression.

%%% \noindent{\bf Assumption $3$}.
\begin{assumption}\label{assumption3} We assume that there exist an exponent $\theta
\in [0, 1]$ and a positive constant $c_\theta$ such that
\begin{eqnarray}
&&\int_{\mathcal Z} \left\{\left(L^*(x, y, f(x)) -L^*(x, y,
f^*_{{\mathcal F}, P}(x)\right)^2\right\} d P(x, y) \nonumber
\\
&\leq& c_\theta \left(1 + \|f\|_\infty\right)^{2-\theta}
\left\{{\mathcal R}_{L^*, P} (f)
-{\mathcal R}_{L^*, P} (f^*_{{\mathcal F}, P})\right\}^\theta,
\quad \forall f\in {\mathcal F}.  \label{varianceexpect}
\end{eqnarray}
\end{assumption}

\begin{remark}
Assumption \ref{assumption3} always holds true for $\theta =0$.
If the triple $(P,{\mathcal F}, L)$ satisfies some conditions, the exponent $\theta$
can be larger. For example, when $L$ is the pinball loss (\ref{pinloss}) and $P$ has
a $\tau$-quantile of $p$-average type $q$ for some $p\in (0,
\infty]$ and $q\in (1, \infty)$ as defined in
\cite{SteinwartChristmann2008a}, then $\theta = \min\{\frac{2}{q},
\frac{p}{p+1}\}$.
\end{remark}

\begin{theorem}\label{mainratesThm}
Suppose that $L(x, y, 0)$ is bounded by a constant $|L|_0$ almost
surely. Under Assumptions \ref{assumption1} to \ref{assumption3}, if we take $\epsilon
>0$ and $\lambda =n^{-\beta}$ for some $\beta >0$, then for any $0<
\delta <1$, with confidence at least $1- \delta$ we have
\begin{equation}\label{explicitrates}
{\mathcal R}_{L^*, P} (f_{L, {\mathbb D}_n, \lambda}) - {\mathcal
R}^*_{L^*, P, {\mathcal F}} \leq \widetilde{C} \left(\log
\frac{2}{\delta} + \log \left(\log \frac{1}{\epsilon} +2\right)
\right)^2 n^{\epsilon- \alpha(r, \beta, \theta, \zeta)},
\end{equation}
where
$\alpha(r, \beta, \theta, \zeta)$ is given by
\begin{eqnarray} \label{explicitratesCZ2}
&  \min\Biggl\{r\beta, \ \frac{1}{2} +
\beta\left(\frac{\theta (1+r)}{4} - \frac{1-r}{2}\right), \
\frac{4}{4-2\theta+ \zeta\theta} - \beta, \qquad \qquad \qquad &  \\
 & ~~~~ \frac{2}{4-2\theta+ \zeta\theta} - \frac{(1-r) \beta}{2}, \
\frac{2}{4-2\theta+ \zeta\theta} - \frac{(1-r) \beta}{2} -
\frac{\beta(1+r)(1- \frac{\theta}{2}) -1}{4}\Biggr\} &  \nonumber
\end{eqnarray} and $\widetilde{C}$ is constant independent of
$n$ or $\delta$ (to be given explicitly in the proof).
\end{theorem}

%-------------------------------------------------------------------------------------

\section{Comparison of learning rates}\label{ComparisonSection}

We now add some theoretical and numerical comparisons on the goodness of our learning rates
with those from the literature.
As already mentioned in the introduction, some reasons for the popularity of additive models are
flexibility, increased interpretability, and (often) a reduced proneness of the curse of high dimensions.
Hence it is important to check, whether the learning rate given in Theorem \ref{mainratesThm} under the
assumption of an additive model
favourably compares to (essentially) optimal learning rates without this assumption.
In other words, we need to demonstrate that the main goal of this paper is achieved by
Theorem \ref{quantileThm} and Theorem \ref{mainratesThm}, i.e. that an SVM based on an additive kernel can provide
a substantially better learning rate in high dimensions than an SVM with a general kernel, say a classical Gaussian RBF kernel,
provided the assumption of an additive model is satisfied.

\begin{remark}
Our learning rate in Theorem \ref{quantileThm} is new and optimal in the literature of SVM for quantile regression. Most learning rates in the literature of SVM for quantile regression are given for projected output functions $\Pi_{|L|_0}(f_{L, {\mathbb D}_n, \lambda})$, while it is well known
that projections improve learning rates \cite{WuYingZhouFoCM}.
Here the projection operator $\Pi_{|L|_0}$ is defined for any measurable function $f: {\mathcal X} \to \RR$ by
\begin{equation}\label{projection}
\Pi_{|L|_0} (f) (x) =\left\{\begin{array}{ll}
f(x), & \hbox{if} \ |f(x)| \leq |L|_0, \\
|L|_0, & \hbox{if} \ f(x) > |L|_0, \\
-|L|_0, & \hbox{if} \ f(x) < - |L|_0.  \end{array}\right. \end{equation}
Sometimes this is called clipping.
Such results are given in \cite{SC2, WuYingZhou}. For example, under the assumptions that $P$ has a
$\tau$-quantile of $p$-average type $2$, the approximation error condition (\ref{approxerrorB})
is satisfied for some $0< r \leq 1$, and that for some constants $a\geq 1, \xi \in (0, 1)$, the
sequence of eigenvalues $(\lambda_i)$ of the integral operator $L_k$ satisfies $\lambda_i \leq a i^{-1/\xi}$ for every $i\in\NN$, it was shown in \cite{SC2} that with confidence at least $1- \delta$,
$$ {\mathcal R}_{L^*, P} \left(\Pi_{|L|_0} (f_{L, {\mathbb D}_n, \lambda})\right) - {\mathcal
R}^*_{L^*, P, {\mathcal F}} \leq \widetilde{C} \log
\frac{2}{\delta} n^{-\alpha}, $$
where
$$ \alpha = \min\left\{\frac{(p+1) r}{(p+2)r +(p+1 - r) \xi}\, , \,\,\frac{2 r}{r+1}\right\}. $$
Here the parameter $\xi$ measures the capacity of the RKHS $H_k$ and it plays a similar role as
half of the parameter $\zeta$ in Assumption 2. For a $C^\infty$ kernel and $r=\frac{1}{2}$, one can
choose $\xi$ and $\zeta$ to be arbitrarily small and the above power index $\alpha$ can be taken
as $\alpha = \min\{\frac{p+1}{p+2}, \frac{2}{3}\} -\epsilon$.

The learning rate in Theorem \ref{quantileThm} may be improved by relaxing Assumption 1
to a Sobolev smoothness condition for $f^*_{{\mathcal F}, P}$ and a regularity condition
for the marginal distribution $P_{\mathcal X}$. For example, one may use a Gaussian kernel
$k=k(n)$ depending on the sample size $n$ and \cite{SteinwartScovel} achieve the
approximation error condition (\ref{approxerrorB}) for some $0< r <1$. This is done for
quantile regression in \cite{Xiang, EbertsSteinwart2013}. Since we are mainly interested in additive
models, we shall not discuss such an extension.
\end{remark}

\begin{example}\label{GaussMore}
Let $s=2$, ${\mathcal X}_1 ={\mathcal X}_2 =[0, 1]$ and ${\mathcal
X} = [0, 1]^2.$ Let $\sigma >0$ and the additive kernel $k$ be given by (\ref{GaussAddForm}) with $k_1, k_2$ in
Example \ref{GaussAdd} as
$$ k_1 (u, v) = k_2 (u, v) =
\exp\left(-\frac{|u-v|^2}{\sigma^2}\right), \qquad u, v\in [0,
1]. $$ If the function $f^*_{{\mathcal F}, P}$ is given by (\ref{gaussfcn}), $|y| \leq |L|_0$ almost surely for some constant
$|L|_0 >0$, and $P$ has a
$\tau$-quantile of $p$-average type $2$ for some $p\in (0,
\infty]$, then by taking $\lambda =n^{-\frac{4(p+1)}{3(p+2)}}$,
for any $\epsilon >0$ and $0<\delta <1$, (\ref{quantilerates}) holds with confidence
at least $1- \delta$.
\end{example}

\begin{remark}
It is unknown whether the above learning rate can be derived by existing approaches
in the literature (e.g. \cite{SC2, SteinwartScovel, WuYingZhou, Xiang, EbertsSteinwart2013})
even after projection. Note that the kernel in the above example is independent of the sample size. 
It would be interesting to see whether there exists some $r>0$ such that the function $f$ defined by
{(\ref{gaussfcn})} lies in the range of the operator $L^r_{k^{\Pi}}$.
The existence of such a positive index would lead to the approximation error condition {(\ref{approxerrorB})}, see \cite{SZII, SunWu}.
\end{remark}

Let us now add some numerical comparisons on the goodness of our learning rates given by Theorem \ref{mainratesThm}
with those given by \cite{EbertsSteinwart2013}. Their Corollary 4.12 gives (essentially) minmax optimal learning
rates for (clipped) SVMs in the context of nonparametric quantile regression using one Gaussian RBF kernel on the whole input space
under appropriate smoothness assumptions of the target function.
Let us consider the case that the distribution $P$ has a $\tau$-quantile of $p$-average type $2$, where
$p=\infty$, and assume that both Corollary 4.12 in \cite{EbertsSteinwart2013} and
our Theorem \ref{mainratesThm} are applicable. I.e.,
we assume in particular that $P$ is a probability measure on
$\mathcal{X}\times\mathcal{Y}:=\R^d\times [-1,+1]$ and that
the marginal distribution $P_\mathcal{X}$ has a Lebesgue density $g\in L_w(\R^d)$ for some $w\ge 1$.
Furthermore, suppose that the optimal decision function $f_{L^*,\mathcal{F},P}^*$ has
(to make Theorem \ref{mainratesThm} applicable with $r\in(0,\frac{1}{2}]$) the additive structure
$f_{L^*,\mathcal{F},P}^*=f_1^*+\ldots+f_s^*$ with each $f_j^*$ as stated in Assumption \ref{assumption1}, where
$\mathcal{X}_j=\R^{d_j}$ and $d:=\sum_{j=1}^s d_j$,
with minimal risk ${\mathcal R}^*_{L^*, P, {\mathcal F}}$ and additionally fulfills
(to make Corollary 4.12 in \cite{EbertsSteinwart2013} applicable)
$$
  f^*_{L^*, P, {\mathcal F}}  \in  L_2(\R^d)\cap L_\infty(\R^d)\cap B_{2s,\infty}^\alpha(\R^d)
$$
where $s:=\frac{w}{w-1}\in[1,\infty]$ and $B_{2s,\infty}^\alpha(\R^d)$ denotes a Besov space
with smoothness parameter $\alpha \ge 1$. The intuitive meaning of $\alpha$ is, that increasing values of $\alpha$ correspond to increased
smoothness.
We refer to \cite[p. 25-27 and p. 44]{EdmundsTriebel} for details on Besov spaces.
It is well-known that the Besov space $B_{p,q}^\alpha(\R^d)$ contains the
Sobolev space $W_p^\alpha(\R^d)$ for $\alpha\in\N$, $p\in(1,\infty)$, and $\max\{p,2\} \le q \le \infty$, and that
$W_2^\alpha(\R^d)=B_{2,2}^\alpha(\R^d)$.
We mention that if all $k_j$ are suitably chosen Wendland kernels,
their reproducing kernel Hilbert spaces $H_j$ are Sobolev spaces, see \cite[Thm. 10.35, p. 160]{Wendland2005}.
Furthermore, we use the same sequence of regularizing parameters as in
\cite[Cor. 4.9, Cor. 4.12]{EbertsSteinwart2013}, i.e.,
\begin{equation}\label{lambdaESsequence}
 \lambda_n = c_1 n^{-\beta_{ES}(d,\alpha,\theta)} \mbox{~, where~}
 \beta_{ES}(d,\alpha,\theta):=\frac{2\alpha+d}{2\alpha(2-\theta)+d}, \quad n\in\N,
\end{equation}
where $d\in\N$, $\alpha\ge 1$, $\theta\in[0,1]$, and
$c_1$ is some user-defined positive constant independent of $n\in\N$. For reasons of simplicity,
let us fix $c_1=1$.
Then \cite[Cor. 4.12]{EbertsSteinwart2013} gives learning rates for the risk of SVMs for $\tau$-quantile regression,
if a single Gaussian RBF-kernel on $\mathcal{X}\subset \R^d$ is used for $\tau$-quantile
functions of $p$-average type $2$ with $p=\infty$, which are of order
$$
c_2 n^{\epsilon-\alpha_{ES}(d,\alpha)}, \mbox{~~where~~} \alpha_{ES}(d,\alpha)=\frac{2\alpha}{2\alpha+d}\, .
$$
Hence the learning rate in Theorem \ref{quantileThm} is better than the one in \cite[Cor. 4.12]{EbertsSteinwart2013} in this situation, if
$$
\alpha(r, \beta_{ES}(d,\alpha,\theta), \theta, \zeta) > \alpha_{ES}(d,\alpha)\, ,
$$
provided the assumption of the additive model is valid.
Table \ref{table1} lists the values of $\alpha(r,\beta_{ES}(d,\alpha,\theta),\theta,\zeta)$ from {(\ref{explicitratesCZ2})}
for some finite values of the dimension $d$, where $\alpha \in [1,\infty)$. All of these values of $\alpha(r,\beta_{ES}(d,\alpha,\theta),\theta,\zeta)$ are positive with
the exceptions if $\theta=0$ or $\zeta\to 2$.
This is in contrast to the corresponding exponent in the learning rate by
\cite[Cor. 4.12]{EbertsSteinwart2013}, because
$$
\lim_{d\to\infty} \alpha_{ES}(d,\alpha)=\lim_{d\to\infty} \frac{2\alpha}{2\alpha+d}=0,
\qquad \forall\,\alpha\in[1,\infty).
$$

Table \ref{table2} and Figures \ref{figure1} to \ref{figure2} give additional information on the limit
$\lim_{d\to\infty} \alpha(r,\beta_{ES}(d,\alpha,\theta),\theta,\zeta)$.
Of course, higher values of the exponent indicates faster rates of convergence.
It is obvious, that an SVM based on an additive kernel has a significantly faster rate of convergence in higher
dimensions $d$ compared to SVM based on a single Gaussian RBF kernel defined on the whole input space,
of course under the assumption that the additive model is valid.
The figures seem to indicate that our learning rate from Theorem \ref{mainratesThm}
is probably not optimal for small dimensions. However, the main focus of the present paper is on high dimensions.

\begin{table}
\begin{center}

\begin{tabular}{rrrr}
 \hline \hline
 $\theta\in[0,1]$ & $\zeta\in(0,2)$ & $\lim_{d\to\infty}\alpha_{ES}(d,\alpha)$  &  $\lim_{d\to\infty}\alpha(r,\beta_{ES}(d,\alpha,\theta),\theta,\zeta)$ \\
                  &                 & from \cite[Cor. 4.12]{EbertsSteinwart2013} & from Thm. \ref{mainratesThm} \\
 \hline
 $>0$ & fixed & $0$ & positive \\
 $1$ & $1$ & $0$ & $\min\{r,1/3\}$ \\
 $1$ & $3/2$ & $0$ & $\min\{r,1/7\}$ \\
 $1/2$ & $1$ & $0$ & $\min\{r,1/7\}$ \\
 $0$ & fixed & $0$ & $0$ \\
 $\in[0,1]$ & $\to 2$ & $0$ & $0$ \\
 \hline
 \hline
 \end{tabular}
\caption{\label{table1} The table lists the limits of the exponents $\lim_{d\to\infty}\alpha_{ES}(d,\alpha)$
from \cite[Cor. 4.12]{EbertsSteinwart2013} and $\lim_{d\to\infty}\alpha(r,\beta_{ES}(d,\alpha,\theta),\theta,\zeta)$
from Theorem \ref{mainratesThm}, respectively,
if the regularizing parameter $\lambda=\lambda_n$ is chosen in an optimal manner for the nonparametric
setup, i.e. $\lambda_n=n^{-\beta_{ES}(d,\alpha,\theta)}$, with $\beta_{ES}(d,\alpha,\theta)\to 1$ for $d\to\infty$ and $\alpha\in[1,\infty)$.
Recall that $r\in (0,\frac{1}{2}]$.}
\end{center}
\end{table}

\begin{table}
\begin{center}
 \begin{tabular}{rrrr}
 \hline \hline
  $r$  &  $\theta$ &  $\zeta$ & $\lim_{d\to \infty} \alpha(r,\beta_{ES}(d,\alpha,\theta),\theta,\zeta)$ \\
\hline
 0.5  &  1  &  0.1  &  0.5   \\
      &     &  1  &  0.333   \\
      &     &  1.9  &  0.026   \\
 \hline
 0.5  &  0.5  &  0.1  &  0.311   \\
      &       &  1  &  0.143   \\
      &       &  1.9  &  0.013   \\
 \hline
 0.5  &  0.1  &  0.1  &  0.05   \\
      &       &  1  &  0.026   \\
      &       &  1.9  &  0.003   \\
\hline
 0.25  &  1  &  0.1  &  0.25   \\
       &     &  1  &  0.25   \\
       &     &  1.9  &  0.026   \\
\hline
 0.25  &  0.5  &  0.1  &  0.25   \\
       &       &  1  &  0.143   \\
       &       &  1.9  &  0.013   \\
\hline
 0.25  &  0.1  &  0.1  &  0.05   \\
       &       &  1  &  0.026   \\
       &       &  1.9  &  0.003   \\
\hline
 0.1  &  1  &  0.1  &  0.1   \\
      &     &  1  &  0.1   \\
      &     &  1.9  &  0.026   \\
\hline
 0.1  &  0.5  &  0.1  &  0.1   \\
      &       &  1  &  0.1   \\
      &       &  1.9  &  0.013   \\
\hline
 0.1  &  0.1  &  0.1  &  0.05   \\
      &       &  1  &  0.026   \\
      &       &  1.9  &  0.003   \\
\hline \hline
\end{tabular}
\caption{\label{table2} The table lists the limits of the exponents $\lim_{d\to\infty} \alpha(r,\beta_{ES}(d,\alpha,\theta),\theta,\zeta)$
from Theorem \ref{mainratesThm},
if the regularizing parameter $\lambda$ is chosen in optimal manner for the nonparametric
setup, i.e. $\lambda=n^{-(2\alpha+d)/(2\alpha(2-\theta)+d)}$ with $\alpha\in[1,\infty)$ and $\theta\in[0,1]$,
see \cite[Cor. 4.12]{EbertsSteinwart2013}.}
\end{center}
\end{table}

\begin{figure}
 \begin{center}
\includegraphics[width=\textwidth]{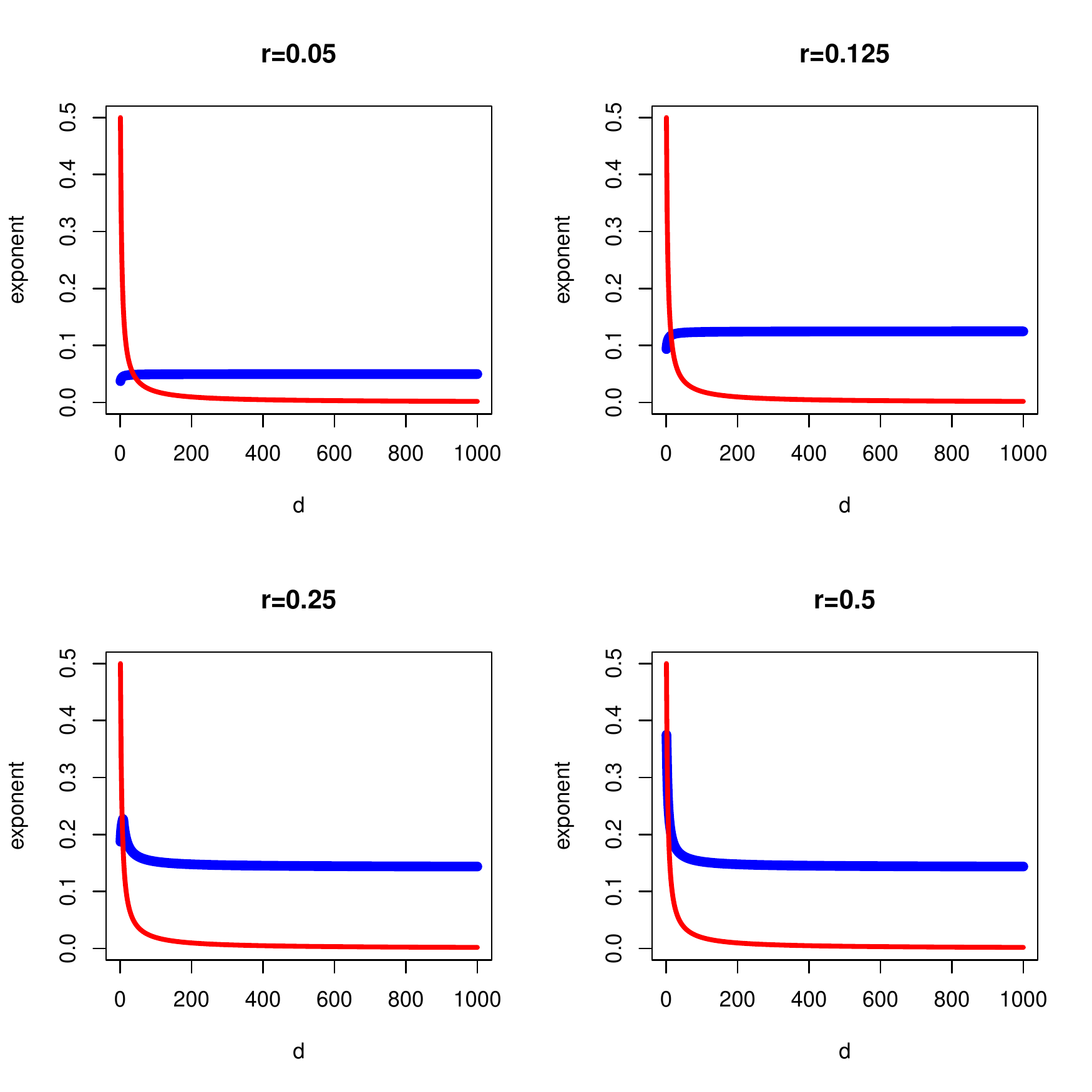}
  \caption{\label{figure1} Plots of the exponents $\lim_{d\to\infty} \alpha(r,\beta_{ES}(d,\alpha,\theta),\theta,\zeta)$
from Theorem \ref{mainratesThm} (thick curve) and \cite[Cor. 4.12]{EbertsSteinwart2013} (thin curve) versus the dimension $d$,
if the regularizing parameter $\lambda=\lambda_n$ is chosen in an optimal manner for the nonparametric
setup, i.e. $\lambda_n=n^{-(2\alpha+d)/(2\alpha(2-\theta)+d)}$ with $\alpha=1$.
We set $\theta=0.5$ and $\zeta=1$.}
\end{center}
\end{figure}

\begin{figure}
 \begin{center}
  \includegraphics[width=\textwidth]{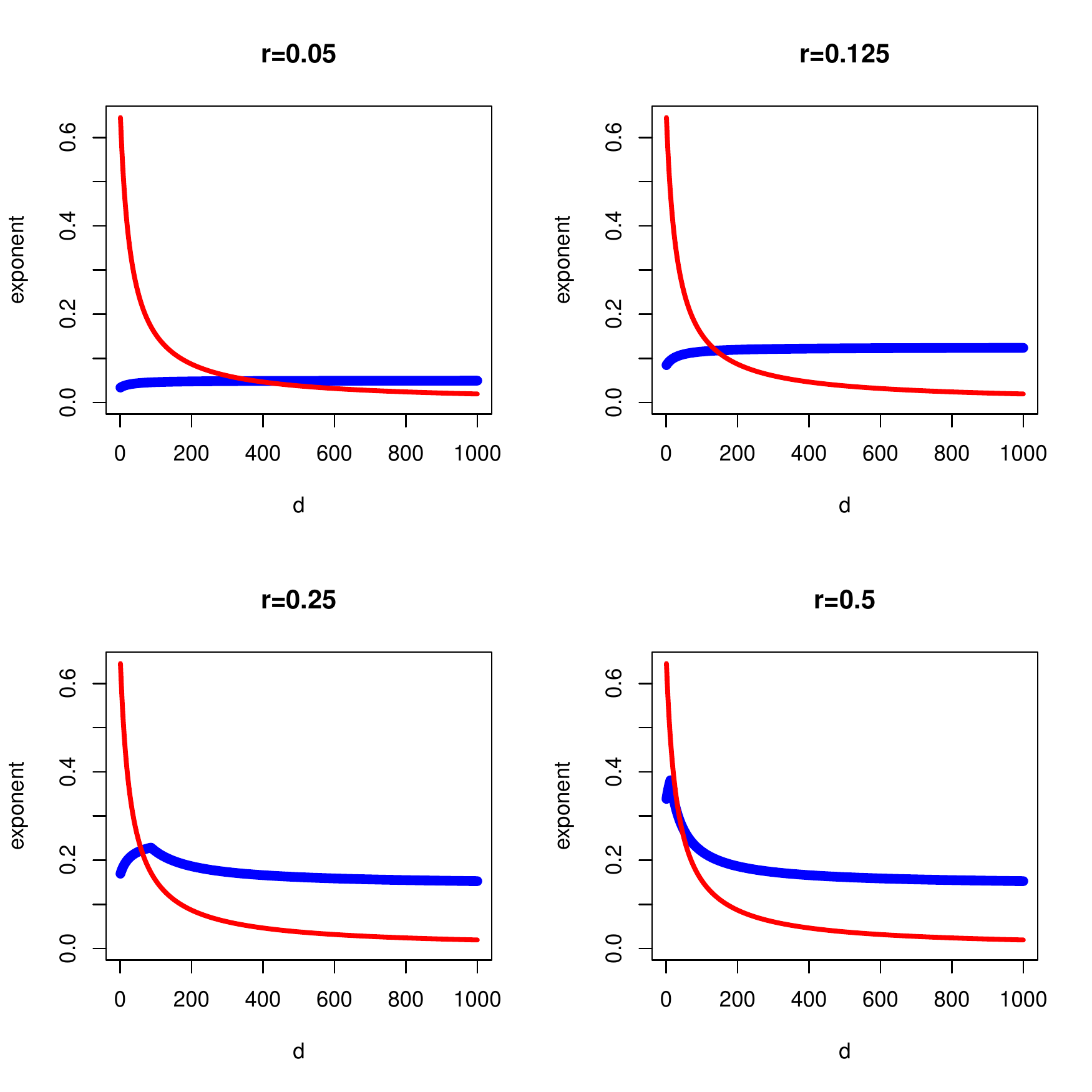}
  \caption{\label{figure2} Similar to Figure \ref{figure1}, but for $\alpha=10$.}
  %Plots of the exponents $\lim_{d\to\infty} \alpha(d)$
  %from Theorem \ref{mainratesThm} (thick curve) and \cite[Cor. 4.12]{EbertsSteinwart2013} (thin curve) versus the dimension $d$,
  %if the regularizing parameter $\lambda=\lambda_n$ is chosen in optimal manner for the nonparametric
  %setup, i.e. $\lambda=n^{-(2\alpha+d)/(2\alpha)(2-\theta)+d)}$ with $\alpha=10$.
  %We set $\theta=0.5$ and $\zeta=1$.}
\end{center}
\end{figure}

We now briefly comment on the goodness of the learning rate provided by Theorem \ref{quantileThm}.
Let us assume that the distribution $P$ on $\mathcal{X}\times\mathcal{Y}:=\R^d\times [-1,+1]$ has a $\tau$-quantile
of $p$-average type $q=2$ for some $p\in(1,\infty]$. Furthermore,
consider the sequence of regularizing parameters
$$
\lambda:= c_1 n^{-\beta_{ES}(d,\alpha,\theta)} \mbox{~, ~with~~}
  \beta_{ES}(d,\alpha,\theta):=\frac{2\alpha+d}{2\alpha(2-\theta)+d} ~,
$$
where $c_1>0$, $\alpha\ge 1$, and $\theta\in[0,1]$.
For reasons of simplicity, we set $c_1=1$.
Under the assumptions of Corollary 4.9 in \cite{EbertsSteinwart2013}, the learning rate for the risk
of SVMs for $\tau$-quantile regression, when a single Gaussian RBF-kernel on $\mathcal{X}=\R^d$
is used, is then of order
$$
  c_2 n^{\epsilon-\alpha_{ES}(d,\alpha,\theta)}, \mbox{~~where~~}
  \alpha_{ES}(d,\alpha,\theta)=\frac{2\alpha}{2\alpha(2-\theta)+d} ~ ,
$$
where $c_2>0$ is a constant independent of $n$.
If $\alpha$, $\theta$, and $p$ are chosen such that
$\frac{2\alpha+d}{2\alpha(2-\theta)+d} = \frac{4(p+1)}{3(p+2)}$ is fulfilled with $d\in\N$,
we can make a fair comparison between the learning rates given  by
\cite[Cor. 4.9]{EbertsSteinwart2013} and by Theorem \ref{quantileThm}, respectively.
Obviously, the learning rate given in Theorem \ref{quantileThm} favourably compares to the one given by
\cite[Cor. 4.9]{EbertsSteinwart2013} for high dimensions $d$, if the assumption of an additive model is satisfied, because the exponent
$\alpha(p)=\frac{2(p+1)}{3(p+2)}$ in Theorem \ref{quantileThm} is positive and independent of $d\in\N$,
whereas $\alpha_{ES}(d,\alpha,\theta)\to 0$, if $d\to\infty$.

Summarizing, the following conclusion seems to be fair. If an additive model is valid and the dimension $d$ of
$X=\R^d$ is high, then it makes sense to use an additive kernel, because
(i) from a theoretical point of view: faster rate of convergence, (ii) from the big data point of view:
the same accuracy of estimating the risk can in principle be achieved already with much smaller data sets,
(iii) from an applied point of view: increased interpretability and flexibility.

%--------------------------------------------------------------------------------------------------------------

\section{Proofs}\label{ProofSection}

This section contains all the proofs of this paper. As some of the results may be interesting in their own, we
treat the topics estimation of the approximation error, the proof of the somewhat surprising assertion in Example \ref{GaussAdd},
sample error estimates, and the proofs of our learning rates from Section \ref{RatesSection} in different subsections.

%--------------------------------------------------------------------------------------------------------------
\subsection{Estimating the approximation error}\label{approximationerror}

To carry out our analysis, we need an error decomposition framework.

\begin{lemma}\label{ErrorDecomThm}
There holds
\begin{equation}\label{ErrorDecomB}
{\mathcal R}_{L^*, P} (f_{L, {\mathbb D}_n, \lambda}) - {\mathcal
R}^*_{L^*, P, {\mathcal F}} + \lambda \|f_{L, {\mathbb D}_n,
\lambda}\|^2_H \leq {\mathcal S} + {\mathcal D}(\lambda),
\end{equation}
where the terms are defined as
\begin{eqnarray}
{\mathcal S} &=& \left\{{\mathcal R}_{L^*, P} (f_{L, {\mathbb
D}_n, \lambda}) - {\mathcal R}_{L^*, {\mathbb D}_n} (f_{L,
{\mathbb D}_n, \lambda})\right\} \nonumber \\
&& + \left\{{\mathcal R}_{L^*, {\mathbb D}_n} (f_{L, P, \lambda}) - {\mathcal R}_{L^*, P} (f_{L, P, \lambda})\right\}, \label{sampleErrorDef} \\
{\mathcal D}(\lambda) &=& {\mathcal R}_{L^*, P} (f_{L, P, \lambda}) - {\mathcal R}^*_{L^*, P, {\mathcal F}} + \lambda \|f_{L, P, \lambda}\|^2_H. \label{approxErrorDef}
\end{eqnarray}
\end{lemma}

\begin{proof}
We compare the risk with the empirical risk and write ${\mathcal R}_{L^*, P} (f_{L, {\mathbb D}_n, \lambda})$ as
$\left\{{\mathcal R}_{L^*, P} (f_{L, {\mathbb D}_n, \lambda}) - {\mathcal R}_{L^*, {\mathbb D}_n} (f_{L, {\mathbb D}_n, \lambda})\right\} + {\mathcal R}_{L^*, {\mathbb D}_n} (f_{L, {\mathbb D}_n, \lambda})$. Then we add and subtract a term involving the function $f_{L, P, \lambda}$ to find
\begin{eqnarray*}
&& {\mathcal R}_{L^*, P} (f_{L, {\mathbb D}_n, \lambda}) - {\mathcal R}^*_{L^*, P, {\mathcal F}} + \lambda \|f_{L, {\mathbb D}_n, \lambda}\|^2_H \\
&=& \left\{{\mathcal R}_{L^*, P} (f_{L, {\mathbb D}_n, \lambda}) - {\mathcal R}_{L^*, {\mathbb D}_n} (f_{L, {\mathbb D}_n, \lambda})\right\} \\
 && + \left\{\left({\mathcal R}_{L^*, {\mathbb D}_n} (f_{L, {\mathbb D}_n, \lambda}) + \lambda \|f_{L, {\mathbb D}_n, \lambda}\|^2_H\right) -
  \left({\mathcal R}_{L^*, {\mathbb D}_n} (f_{L, P, \lambda}) + \lambda \|f_{L, P, \lambda}\|^2_H\right)\right\} \\
&& + \left\{{\mathcal R}_{L^*, {\mathbb D}_n} (f_{L, P, \lambda}) - {\mathcal R}_{L^*, P} (f_{L, P, \lambda})\right\} \\
&& + \left\{{\mathcal R}_{L^*, P} (f_{L, P, \lambda}) - {\mathcal R}^*_{L^*, P, {\mathcal F}} + \lambda \|f_{L, P, \lambda}\|^2_H\right\}.
\end{eqnarray*}
But $\left({\mathcal R}_{L^*, {\mathbb D}_n} (f_{L, {\mathbb D}_n, \lambda}) + \lambda \|f_{L, {\mathbb D}_n, \lambda}\|^2_H\right) -
  \left({\mathcal R}_{L^*, {\mathbb D}_n} (f_{L, P, \lambda}) + \lambda \|f_{L, P, \lambda}\|^2_H\right) \leq 0$ by the definition of $f_{L, {\mathbb D}_n, \lambda}$. Then the desired statement is proved.
\end{proof}

In the error decomposition (\ref{ErrorDecomB}), the first term
${\mathcal S}$ is called \emph{sample error} and will be dealt with
later on. The second term ${\mathcal D}(\lambda)$ is the
\emph{approximation error} which can be stated equivalently by Definition \ref{Defapprox}.

In this section we estimate the approximation error based on
Assumption \ref{assumption1}. Our estimation is based on the following lemma which
is proved by the same method as that in \cite{SZII}. Recall that the integral operator $L_{k_j}$ is a positive operator on $L_2 (P_{{\mathcal X}_j})$, hence $L_{k_j} + \lambda I$ is invertible.

\begin{lemma}\label{intermediatefcn}
Let $j\in \{1, \ldots, s\}$ and $0<r \leq \frac{1}{2}$. Assume $f^*_j = L_{k_j}^r (g^*_j)$ for some $g^*_j \in L_2 (P_{{\mathcal X}_j})$.
Define an intermediate function $f_{j, \lambda}$ on ${\mathcal X}_j$ by
\begin{equation}\label{fjlambda}
 f_{j, \lambda} = (L_{k_j} + \lambda I)^{-1} L_{k_j} (f^*_j).
\end{equation}
Then we have
\begin{equation}\label{flambda-}
 \|f_{j, \lambda} -f^*_j \|^2_{L_2 (P_{{\mathcal X}_j})} + \lambda \|f_{j, \lambda}\|_{k_j}^2 \leq \lambda^{2 r} \|g^*_j\|^2_{L_2 (P_{{\mathcal X}_j})}.
\end{equation}
\end{lemma}

\begin{proof}
If $\{(\lambda_i, \psi_i)\}_{i\ge 1}$ are the normalized eigenpairs of
the integral operator $L_{k_j}$, then the system
$\{\sqrt{\lambda_i} \psi_i: \lambda_i >0\}$ is orthogonal in $H_j$.

Write $g^*_j=\sum_{i\ge 1} d_i \psi_i$ with
$\|\{d_i\}\|_{\ell^2} =\|g^*_j\|_{L_2 (P_{{\mathcal X}_j})} <\infty$. Then $f^*_j
=\sum_{i\ge 1} \lambda_i^r  d_i \psi_i$ and
$$ f_{j, \lambda} -f^*_j =\bigl(L_{k_j} + \lambda I\big)^{-1} L_{k_j} (f^*_j) - f^*_j
=-\sum_{i\ge 1} \frac{\lambda}{\lambda_i + \lambda} \lambda_i^r  d_i \psi_i. $$ Hence
$$\|f_{j, \lambda} -f^*_j\|^2_{L_2 (P_{{\mathcal X}_j})} =\sum_{i\ge 1}
\biggl(\frac{\lambda}{\lambda_i + \lambda} \lambda_i^r d_i
 \biggr)^2 =\lambda^{2 r} \sum_{i\ge 1}
\biggl(\frac{\lambda}{\lambda_i + \lambda}\biggr)^{2(1-r)} \biggl(\frac{\lambda_i}{\lambda_i + \lambda}\biggr)^{2 r} d_i^2. $$
Also,
$$\|f_{j, \lambda}\|_{k_j}^2 = \left\|\sum_{i\ge 1} \frac{\lambda_i}{\lambda_i + \lambda} \lambda_i^r  d_i \psi_i\right\|_{k_j}^2
= \left\|\sum_{i\ge 1} \frac{\lambda_i^{\frac{1}{2} + r}}{\lambda_i + \lambda}  d_i \sqrt{\lambda_i} \psi_i\right\|_{k_j}^2 =
\sum_{i\ge 1} \frac{\lambda_i^{1 + 2 r}}{(\lambda_i + \lambda)^2}  d_i^2.  $$
Therefore, we have
\begin{eqnarray*}
&& \|f_{j, \lambda} -f^*_j \|^2_{L_2 (P_{{\mathcal X}_j})} + \lambda \|f_{j, \lambda}\|_{k_j}^2 \\
&=&\lambda^{2 r} \sum_{i\ge 1}
\left\{\biggl(\frac{\lambda}{\lambda_i + \lambda}\biggr)^{2(1-r)} \biggl(\frac{\lambda_i}{\lambda_i + \lambda}\biggr)^{2 r}
+ \biggl(\frac{\lambda}{\lambda_i + \lambda}\biggr)^{1-2r} \biggl(\frac{\lambda_i}{\lambda_i + \lambda}\biggr)^{1 + 2 r} \right\} d_i^2 \\
&\leq& \lambda^{2 r} \sum_{i\ge 1}
\left\{\frac{\lambda}{\lambda_i + \lambda}
+ \frac{\lambda_i}{\lambda_i + \lambda} \right\} d_i^2 = \lambda^{2 r} \|\{d_i\}\|^2_{\ell^2} = \lambda^{2 r} \|g^*_j\|^2_{L_2 (P_{{\mathcal X}_j})}.
\end{eqnarray*}
This proves the desired bound.
\end{proof}

%------------------------------------------------------------------------------------
\subsection{Proof of Theorem \ref{approxerrorThm}}

\proof[Proof of Theorem \ref{approxerrorThm}]
Observe that $f_{j, \lambda} \in H_j$. So $f_{1, \lambda} + \ldots + f_{s, \lambda} \in H$ and by the definition of the approximation error, we have
$$ {\mathcal D}(\lambda) \leq {\mathcal R}_{L^*, P} (f_{1, \lambda} + \ldots + f_{s, \lambda}) - {\mathcal R}^*_{L^*, P, {\mathcal F}} + \lambda \|f_{1, \lambda} + \ldots + f_{s, \lambda}\|_H^2. $$
But
$$ {\mathcal R}^*_{L^*, P, {\mathcal F}} =  {\mathcal R}_{L^*, P}(f^*_{{\mathcal F}, P}) ={\mathcal R}_{L^*, P}(f^*_1 + \ldots + f^*_s)$$
according to Assumption \ref{assumption1}.
Using the inequality in {(\ref{ACNormInequality})}, we obtain
$$  {\mathcal D}(\lambda)
\leq
{\mathcal R}_{L^*, P} (f_{1, \lambda} + \ldots + f_{s, \lambda})
- {\mathcal R}_{L^*, P}(f^*_1 + \ldots + f^*_s)
+ \lambda \sum_{j=1}^s \|f_{j, \lambda}\|^2_{H_j} .
$$
Applying the Lipschitz property (\ref{Lipsch}), the excess risk term can be estimated as
\begin{eqnarray*}
&& {\mathcal R}_{L^*, P} (f_{1, \lambda} + \ldots + f_{s, \lambda}) - {\mathcal R}_{L^*, P}(f^*_1 + \ldots + f^*_s) \\
&=& \int_{{\mathcal Z}} L^* \left(x, y, f_{1, \lambda}(x_1) + \ldots + f_{s, \lambda}(x_s)\right) d P (x, y) \\
&& - \int_{{\mathcal Z}} L^* \left(x, y, f^*_1 (x_1) + \ldots + f^*_s (x_s)\right) d P (x, y) \\
& \leq & \int_{{\mathcal Z}} |L|_1 \Bigl| \sum_{j=1}^s f_{j, \lambda}(x_j)
- \sum_{j=1}^s f^*_j (x_j) \Bigr| d P (x, y) \\
& \leq & |L|_1 \sum_{j=1}^s \int_{{\mathcal X}_j} \left|f_{j, \lambda}(x_j) - f^*_j (x_j)\right| d P_{{\mathcal X}_j} (x_j) \,.
\end{eqnarray*}
But
$$\int_{{\mathcal X}_j} \left|f_{j, \lambda}(x_j) - f^*_j (x_j)\right| d P_{{\mathcal X}_j} (x_j) = \|f_{j, \lambda} - f^*_j\|_{L_1 (P_{{\mathcal X}_j})} \leq  \|f_{j, \lambda} - f^*_j\|_{L_2 (P_{{\mathcal X}_j})}. $$
\noindent
The bound {(\ref{flambda-})} implies the following two inequalities
\begin{equation}\label{ProofThm1F1}
  \|f_{j, \lambda} -f^*_j \|^2_{L_2 (P_{{\mathcal X}_j})} \leq \lambda^{2 r} \|g^*_j\|^2_{L_2 (P_{{\mathcal X}_j})}
\end{equation}
and
\begin{equation}\label{ProofThm1F2}
  \lambda \|f_{j, \lambda}\|_{k_j}^2 \leq \lambda^{2 r} \|g^*_j\|^2_{L_2 (P_{{\mathcal X}_j})}.
\end{equation}
Taking square roots on both sides in {(\ref{ProofThm1F1})} yields
$$  {\mathcal D}(\lambda) \leq \sum_{j=1}^s \left(|L|_1 \|f_{j, \lambda} - f^*_j\|_{L_2 (P_{{\mathcal X}_j})} + \lambda \|f_{j, \lambda}\|^2_{H_j}\right).$$
This together with {(\ref{ProofThm1F2})} and Lemma \ref{intermediatefcn} gives
$$  {\mathcal D}(\lambda) \leq \sum_{j=1}^s \left\{|L|_1 \lambda^{r} \|g^*_j\|_{L_2 (P_{{\mathcal X}_j})} + \lambda^{2 r} \|g^*_j\|^2_{L_2 (P_{{\mathcal X}_j})}\right\} $$
and completes the proof of the statement.
\qed

%--------------------------------------------------------------------------------------------------------------
\subsection{Proof of the assertion in Example \ref{GaussAdd}}\label{ProofSectionIntroduction}

%We can now prove the conclusion in Example \ref{GaussAdd}.

\proof[Proof of Example \ref{GaussAdd}]
The function $f$ can be written as $f =f_1 +0$ where $f_1$ is a
function on ${\mathcal X}_1$ given by $f_1 (x_1, x'_1) =k_1 (x_1,
0) \in H_1$. So $f\in H$.

Now we prove (\ref{notinclude}). Assume to the contrary that $f
\in H_{k^{\Pi}}$. We apply a characterization of the RKHS
$H_{k^{\Pi}}$ given in \cite[Thm. 1]{Minh} as
\begin{equation}\label{charactGauss}
H_{k^{\Pi}} =\left\{f = e^{-\frac{\|x\|^2}{\sigma^2}}
\sum_{|\alpha| =0}^\infty w_\alpha x^\alpha: \|f\|_K^2 =
\sum_{\ell =0}^\infty \frac{\ell!}{(2/\sigma^2)^\ell}
\sum_{|\alpha| = \ell} \frac{w_\alpha^2}{C^\ell_\alpha} <
\infty\right\},
\end{equation}
where $\|x\|^2 = |x_1|^2 + |x_2|^2$ and $C^\ell_\alpha =
\frac{\ell!}{\alpha_1! \alpha_2!}$ for $\alpha = (\alpha_1,
\alpha_2) \in \ZZ_+^2$. Since $f \in H_{k^{\Pi}}$, we have
$$f\left(x_1, x_2\right) = \exp\left\{-\frac{|x_1|^2}{\sigma^2}\right\}
=e^{-\frac{|x_1|^2 + |x_2|^2}{\sigma^2}} \sum_{|\alpha| =0}^\infty
w_\alpha x^\alpha $$ where the coefficient sequence $\{w_\alpha:
\alpha \in \ZZ_+^2\}$ satisfies
$$ \|f\|_K^2 = \sum_{\ell =0}^\infty
\frac{\ell!}{(2/\sigma^2)^\ell} \sum_{|\alpha| = \ell}
\frac{w_\alpha^2}{C^\ell_\alpha} < \infty. $$ It follows that
$$\exp\left\{\frac{|x_2|^2}{\sigma^2}\right\} = \sum_{m=0}^\infty
\frac{1}{m!} \left(\frac{|x_2|^2}{\sigma^2}\right)^m
=\sum_{|\alpha| =0}^\infty w_\alpha x^\alpha. $$ Hence
$$ w_\alpha =\left\{\begin{array}{ll} \frac{1}{m! \sigma^{2 m}}, &
\hbox{if} \ \alpha =(0, 2m) \ \hbox{with} \ m\in\ZZ_+, \\
0, & \hbox{otherwise,} \end{array}\right. $$ and
$$ \|f\|_K^2 = \sum_{m =0}^\infty
\frac{(2 m)!}{(2/\sigma^2)^{2 m}} \frac{w_{(0, 2m)}^2}{C^{2
m}_{(0, 2m)}} = \sum_{m =0}^\infty \frac{(2 m)!}{(2/\sigma^2)^{2
m}} \left(\frac{1}{m! \sigma^{2 m}}\right)^2 = \sum_{m =0}^\infty
\frac{(2 m)!}{2^{2 m} (m!)^2}.
$$
Finally we apply the Stirling's approximation:
$$ \sqrt{2 \pi m} \left(\frac{m}{\pi}\right)^m \leq m! \leq
\frac{e}{\sqrt{2 \pi}} \sqrt{2 \pi m}
\left(\frac{m}{\pi}\right)^m,
$$
and find
$$ \|f\|_K^2 = \sum_{m =0}^\infty
\frac{(2 m)!}{2^{2 m} (m!)^2} \geq \sum_{m =0}^\infty
\frac{\sqrt{2 \pi (2 m)} \left(\frac{2 m}{\pi}\right)^{2 m}}{2^{2
m} \left(\frac{e}{\sqrt{2 \pi}} \sqrt{2 \pi m}
\left(\frac{m}{\pi}\right)^m\right)^2} =\sum_{m =0}^\infty
\frac{2\sqrt{\pi}}{e^2 \sqrt{m}} =\infty.
$$
This is a contradiction. Therefore, $f \not\in H_{k^{\Pi}}$. This
proves the conclusion in Example \ref{GaussAdd}.
\qed

%--------------------------------------------------------------------------------------------------------------
\subsection{Sample error estimates}\label{SampleSection}

In this subsection we bound the sample error ${\mathcal S}$ defined
by (\ref{sampleErrorDef}) by Assumption \ref{assumption3}. It can first
be decomposed in two terms:
\begin{equation}\label{sampleDecom}
{\mathcal S} = {\mathcal S}_1 + {\mathcal S}_2,
\end{equation}
where
\begin{eqnarray}
{\mathcal S}_1 &=& \left\{{\mathcal R}_{L^*, P} (f_{L, {\mathbb
D}_n, \lambda}) -{\mathcal R}_{L^*, P} (f^*_{{\mathcal F},
P})\right\} \nonumber \\
&& \qquad - \left\{{\mathcal R}_{L^*, {\mathbb D}_n} (f_{L,
{\mathbb D}_n, \lambda}) - {\mathcal R}_{L^*, {\mathbb D}_n}
(f^*_{{\mathcal F}, P}) \right\}, \label{Sone} \\
{\mathcal S}_2 &=& \left\{{\mathcal R}_{L^*, {\mathbb D}_n} (f_{L,
P, \lambda}) - {\mathcal R}_{L^*, {\mathbb D}_n} (f^*_{{\mathcal
F}, P})\right\} \nonumber \\
&& \qquad - \left\{{\mathcal R}_{L^*, P} (f_{L, P, \lambda}) -
{\mathcal R}_{L^*, P} (f^*_{{\mathcal F}, P})\right\}.
\label{Stwo}
\end{eqnarray}

The second term ${\mathcal S}_2$ can be bounded easily by the Bernstein inequality.

\begin{lemma}\label{S2B}
Under Assumptions \ref{assumption1} and \ref{assumption3}, for any $0< \lambda \leq 1$ and $0< \delta <1$, with confidence $1-\frac{\delta}{2}$, we have
\begin{equation}\label{S2Bound}
{\mathcal S}_2 \leq C'_1 \log \frac{2}{\delta}
\max\left\{\frac{\lambda^{\frac{r-1}{2}}}{n},
\frac{\lambda^{\frac{r-1}{2} +
\frac{\theta(r+1)}{4}}}{\sqrt{n}}\right\},
\end{equation}
where $C'_1$ is a constant independent of $\delta, n$ or
$\lambda$ and given explicitly in the proof, see {(\ref{ACLemma3formula1})}.
\end{lemma}

\begin{proof}
Consider the random variable $\xi$ on $({\mathcal Z}, \mathcal{B}(\mathcal{Z}))$
defined by
$$ \xi (x, y) = L^* (x, y, f_{L, P, \lambda}(x)) - L^* (x, y, f^*_{{\mathcal
F}, P}(x)), \qquad z=(x, y) \in {\mathcal Z}. $$
Here $\mathcal{B}(\mathcal{Z})$ denotes the Borel-$\sigma$-algebra.
Recall our notation for the constant $\kappa:=\sum_{j=1}^s \sup_{x_j \in {\mathcal X}_j} \sqrt{k_j (x_j, x_j)} \geq \sqrt{\|k\|_\infty}$.
By Assumption \ref{assumption1} and Assumption \ref{assumption3}, $\|f^*_{{\mathcal F}, P}\|_{L_\infty (P_{{\mathcal X}})} < \infty$ and by Theorem \ref{approxerrorThm},
$$ \|f_{L, P, \lambda}\|_{L_\infty (P_{{\mathcal X}})} \leq \kappa \|f_{L, P, \lambda}\|_{H}
\leq \kappa \sqrt{{\mathcal D}(\lambda)/\lambda} \leq \kappa \sqrt{C_r} \lambda^{\frac{r-1}{2}} ~ < ~ \infty. $$
This in connection with the Lipschitz condition (\ref{Lipsch}) for $L$ tells us that the random
variable $\xi$ is bounded by
$$
B_\lambda := |L|_1 \left(\|f^*_{{\mathcal F}, P}\|_{L_\infty (P_{{\mathcal X}})} + \kappa \sqrt{C_r} \lambda^{\frac{r-1}{2}}\right).
$$

By Assumption \ref{assumption3}, we also know that its variance $\sigma^2 (\xi)$ can be bounded as
\begin{eqnarray*}
\sigma^2 (\xi) &\leq& \int_{\mathcal Z} \left(\xi (x, y)\right)^2 d P(x, y) \\
&\leq& c_\theta \left(1 + \|f_{L, P, \lambda}\|_{L_\infty (P_{{\mathcal X}})}\right)^{2-\theta}
\left\{{\mathcal R}_{L^*, P} (f_{L, P, \lambda})
-{\mathcal R}_{L^*, P} (f^*_{{\mathcal F}, P})\right\}^\theta \\
&\leq& c_\theta \left(1 + \kappa \sqrt{C_r} \lambda^{\frac{r-1}{2}}\right)^{2-\theta}
\left\{C_r \lambda^r \right\}^\theta \leq c_\theta \left(1 + \kappa \sqrt{C_r}\right)^{2-\theta}
C_r^\theta  \lambda^{r-1 + \frac{\theta(r+1)}{2}}.
\end{eqnarray*}
Now we apply the one-sided Bernstein inequality to $\xi$ which asserts that, for all $\epsilon>0$,
$$ \hbox{Prob} \biggl({1\over n}\sum_{i=1}^n \xi (z_i) - \E (\xi) > \epsilon \biggr) \le
\exp \biggl(-\frac{n \epsilon^2}{2\bigl(\sigma^2 (\xi) +
\frac{1}{3} B_\lambda \epsilon\bigr)}\biggr). $$
Solving the quadratic equation
$$\frac{n \epsilon^2}{2\bigl(\sigma^2 (\xi) +{1 \over 3} B_\lambda \epsilon\bigr)} = \log \frac{2}{\delta} $$ for
$\epsilon>0$, we see that with confidence $1-\frac{\delta}{2}$, we
have
\begin{eqnarray*} && {1\over n}\sum_{i=1}^n \xi (z_i) - \E (\xi)
\le \frac{{1 \over 3} B_\lambda \log \frac{2}{\delta} +
\sqrt{\bigl(\frac{1}{3} B_\lambda
\log \frac{2}{\delta} \bigr)^2 + 2 n \sigma^2 (\xi) \log \frac{2}{\delta}}}{n} \\
 && \qquad \leq \frac{2 B_\lambda \log \frac{2}{\delta}}{3n}  +\sqrt{\frac{2\log \frac{2}{\delta}}{n} \sigma^2 (\xi)}
 \leq C'_1 \log \frac{2}{\delta} \max\left\{\frac{\lambda^{\frac{r-1}{2}}}{n}, \frac{\lambda^{\frac{r-1}{2} + \frac{\theta(r+1)}{4}}}{\sqrt{n}}\right\},
\end{eqnarray*}
where $C'_1$ is the constant given by
\begin{equation}  \label{ACLemma3formula1}
 C'_1 = |L|_1 \left(\|f^*_{{\mathcal F}, P}\|_{L_\infty (P_{{\mathcal X}})} + \kappa \sqrt{C_r}\right) + \sqrt{2 c_\theta} \left(1 + \kappa \sqrt{C_r}\right)^{1-\frac{\theta}{2}}
C_r^{\frac{\theta}{2}}.
\end{equation}
But ${1\over n}\sum_{i=1}^n \xi (z_i) -
\E (\xi) = {\mathcal S}_2$. So our conclusion follows.
\end{proof}

The term ${\mathcal S}_1$ involves the function $f_{L, {\mathbb
D}_n, \lambda}$ which varies with the sample. Hence we need a concentration inequality to bound this term.
We shall do so by applying the following concentration inequality \cite{WuYingZhou} to the function set
\begin{equation}\label{hypothesisS}
 {\mathcal G} =\left\{L^* (x, y, f(x)) - L^* (x, y, f^*_{{\mathcal
F}, P}(x)): \ f\in H \mbox{~with~} \|f\|_H \leq R\right\}
\end{equation}
parameterized by the radius $R$ involving the $\ell_2$-empirical covering numbers of the function set.

\begin{prop}\label{concenProp}\cite[Prop.\,6]{WuYingZhou}
Let ${\mathcal G}$ be a set of measurable
functions on ${\mathcal Z}$, and $B, c>0, \theta \in [0, 1]$ be constants such
that each function $f\in {\mathcal G}$ satisfies $\|f\|_\infty \le B$
and $\E (f^2)\le c (\E f)^\theta$. If for some $a >0$ and $p\in (0, 2)$,
$$\sup_{\ell \in\NN} \sup_{{\bf z} \in {\mathcal Z}^\ell}\log {\cal N}_{2, {\bf z}}({\cal G}, \epsilon)\le
a\epsilon^{-p},\qquad \forall \epsilon >0, $$ then there
exists a constant $c_p'$ depending only on $p$ such that for any
$t>0$, with probability at least $1-e^{-t}$, there holds
$$ \E f- {1\over n}\sum_{i=1}^n
f(z_i)\le {1\over 2}\eta^{1-\theta}(\E f)^\theta+ c_p' \eta +
2\Big({c t \over n}\Big)^{1/(2-\theta)}+{18 B t \over n}, \qquad
\forall f\in{\cal G}, $$ where
$$\eta :=\max\bigg\{c^{2-p\over
4-2\theta +p\theta}\Big( \displaystyle{a\over n}\Big)^{2\over
4-2\theta+p\theta}, \ B^{2-p\over 2+p}\Big(\displaystyle { a\over
n}\Big)^{2\over 2+p}\bigg\}.$$
\end{prop}

\begin{lemma}\label{S1BR}
Under Assumptions \ref{assumption2} and \ref{assumption3}, for any $R \geq 1$, $0< \lambda \leq 1$ and $0< \delta <1$, with confidence $1-\frac{\delta}{2}$, we have
\begin{eqnarray}
&& \left\{{\mathcal R}_{L^*, P} (f) -{\mathcal R}_{L^*, P} (f^*_{{\mathcal F}, P})\right\} - \left\{{\mathcal R}_{L^*, {\mathbb D}_n} (f) - {\mathcal R}_{L^*, {\mathbb D}_n}
(f^*_{{\mathcal F}, P})\right\} \nonumber \\
&\leq& C'_2 R^{1-\theta} n^{-\frac{2(1-\theta)}{4-2\theta+ \zeta\theta}} \left({\mathcal R}_{L^*, P} (f) -{\mathcal R}_{L^*, P} (f^*_{{\mathcal F}, P})\right)^\theta \nonumber \\
&& + C''_2 \log \frac{2}{\delta} R n^{-\frac{2}{4-2\theta+
\zeta\theta}}, \qquad \forall \|f\|_H \leq R, \label{S2Bound}
\end{eqnarray}
where $C'_2, C''_2$ are constants independent of $R, \delta, n$ or
$\lambda$ and given explicitly in the proof. In particular, $C'_2
=\frac{1}{2}$ when $\theta =1$.
\end{lemma}

\begin{proof}
Consider the function set ${\mathcal G}$ defined by (\ref{hypothesisS}). Each function takes the form
$g(x, y) = L^* (x, y, f(x)) - L^* (x, y, f^*_{{\mathcal F}, P}(x))$ with $\|f\|_H \leq R$. It satisfies
$$ \|g\|_\infty \leq |L|_1 \|f- f^*_{{\mathcal F}, P}\|_\infty \leq |L|_1 \left(\kappa + \|f^*_{{\mathcal F}, P}\|_{L_\infty (P_{{\mathcal X}})}\right) R =: B $$
and by Assumption \ref{assumption3} and the condition $R \geq 1$,
$$ \E (g^2) \leq (1+ \kappa)^{2-\theta} c_\theta R^{2 - \theta} (\E g)^\theta. $$
Moreover, the Lipschitz  property (\ref{Lipsch}) and Lemma \ref{capacityThm} imply that for any $\epsilon >0$ there holds
$$\sup_{\ell \in\NN} \sup_{{\bf z} \in {\mathcal Z}^\ell}\log {\cal N}_{2, {\bf z}}({\cal G}, \epsilon) \le \log {\cal N}\left(\{f\in H: \|f\|_{H} \leq R\}, \frac{\epsilon}{|L|_1}\right)
\leq s c_\zeta \left(\frac{s |L|_1 R}{\epsilon}\right)^\zeta. $$
Thus all the conditions of Proposition \ref{concenProp} are satisfied with $p=\zeta$ and we see that with confidence at least $1-\frac{\delta}{2}$, there holds
\begin{eqnarray}
&& \E g - {1\over n}\sum_{i=1}^n g (z_i)\le {1\over
2}\eta^{1-\theta}(\E g)^\theta+ c'_\zeta \eta +
2\Big({c \log(2/\delta) \over n}\Big)^{1/(2-\theta)} \nonumber \\
&& \quad +{18 B \log(2/\delta) \over n},  \qquad \forall g\in{\cal
G}, \label{concenapply}
\end{eqnarray}
where $c= (1+ \kappa)^{2-\theta} c_\theta R^{2 - \theta}$, $a= s c_\zeta \left(s |L|_1 R\right)^\zeta$ and
$$\eta =\max\bigg\{c^{2- \zeta \over
4-2\theta +\zeta\theta}\Big( \displaystyle{a\over n}\Big)^{2\over
4-2\theta+ \zeta\theta}, \ B^{2- \zeta \over 2+
\zeta}\Big(\displaystyle { a\over n}\Big)^{2\over 2+
\zeta}\bigg\}.$$ But
$$\E g = {\mathcal R}_{L^*, P} (f) -{\mathcal R}_{L^*, P} (f^*_{{\mathcal F},
P}) $$
and
$$ {1\over n}\sum_{i=1}^n
g (z_i) = {\mathcal R}_{L^*, {\mathbb D}_n} (f) - {\mathcal
R}_{L^*, {\mathbb D}_n} (f^*_{{\mathcal F}, P}). $$ Notice from
the inequality $4-2\theta+ \zeta\theta \geq 2+ \zeta$ that
$$ \eta \leq C'_3 R n^{-\frac{2}{4-2\theta+ \zeta\theta}}, $$
where $C'_3$ is the constant given by
\begin{eqnarray*} C'_3 :&=& ((1+ \kappa)^{2-\theta} c_\theta)^{2- \zeta \over
4-2\theta +\zeta\theta} \left(s c_\zeta (s |L|_1)^\zeta\right)^{2\over
4-2\theta+ \zeta\theta} \\
&& + \left(|L|_1 \left(\kappa + \|f^*_{{\mathcal F}, P}\|_{L_\infty (P_{{\mathcal X}})}\right)\right)^{2- \zeta \over 2+ \zeta} \left(s c_\zeta (s |L|_1)^\zeta\right)^{2\over 2+ \zeta}.
\end{eqnarray*}
Then our desired bound holds true with the constants given by
$$ C''_2 = \max\left\{{1\over 2} (C'_3)^{1-\theta}, c'_\zeta C'_3, 2(1+ \kappa) (c_\theta)^{1/(2-\theta)} + 18 |L|_1 \left(\kappa + \|f^*_{{\mathcal F}, P}\|_{L_\infty (P_{{\mathcal X}})}\right)\right\} $$
and
$$ C'_2 = \left\{\begin{array}{ll} C''_2, & \hbox{if} \ 0\leq \theta <1, \\ \frac{1}{2}, & \hbox{if} \ \theta =1. \end{array}\right. $$
Here the case $\theta =1$ can be seen directly from (\ref{concenapply}).
This completes the proof.
\end{proof}

Combining all the above results yields the following error bounds. For $R \geq 1$, we denote a sample set
\begin{equation}\label{WRset}
{\mathcal W}(R) = \left\{{\bf z} \in {\mathcal Z}^n: \|f_{L,
{\mathbb D}_n, \lambda}\|_H \leq R\right\}.
\end{equation}

\begin{prop}\label{generalThm}
Under Assumptions \ref{assumption1} to \ref{assumption3}, let $R \geq 1$, $0< \lambda \leq 1$
and $0< \delta <1$. Then there exists a subset ${\mathcal V}_R$ of
${\mathcal Z}^n$ with probability at most $\delta$ such that
\begin{eqnarray}
&& {\mathcal R}_{L^*, P} (f_{L, {\mathbb D}_n, \lambda}) - {\mathcal
R}^*_{L^*, P, {\mathcal F}} + \lambda \|f_{L, {\mathbb D}_n,
\lambda}\|^2_H \nonumber \\
&& \leq C'_2 R^{1-\theta} n^{-\frac{2(1-\theta)}{4-2\theta+ \zeta\theta}} \left({\mathcal R}_{L^*, P} (f_{L, {\mathbb D}_n, \lambda}) -{\mathcal R}_{L^*, P} (f^*_{{\mathcal F}, P})\right)^\theta \nonumber \\
&& \quad  + C'_1 \log \frac{2}{\delta} \max\left\{\frac{\lambda^{\frac{r-1}{2}}}{n}, \frac{\lambda^{\frac{r-1}{2} + \frac{\theta(r+1)}{4}}}{\sqrt{n}}\right\} \nonumber \\
&& \quad  + C''_2 \log \frac{2}{\delta} R n^{-\frac{2}{4-2\theta+
\zeta\theta}} + C_r \lambda^{r}, \qquad \forall {\bf z} \in
{\mathcal W}(R)\setminus {\mathcal V}_R. \label{GeneralB}
\end{eqnarray}
\end{prop}

To apply the above analysis we need a radius $R$ which bounds the norm of the function $f_{L, {\mathbb
D}_n, \lambda}$.

\begin{lemma}\label{initialB}
If $L(x, y, 0)$ is bounded by a constant $|L|_0$ almost surely, then we have almost surely
$$ \|f_{L, {\mathbb D}_n, \lambda}\|_H \leq \sqrt{|L|_0/\lambda}. $$
\end{lemma}

\begin{proof} By the definition of the function $f_{L, {\mathbb
D}_n, \lambda}$, we have
$$ {\mathcal R}_{L^*, {\mathbb D}_n} (f_{L,
{\mathbb D}_n, \lambda}) + \lambda \|f_{L, {\mathbb D}_n, \lambda}\|_H^2 \leq {\mathcal R}_{L^*, {\mathbb D}_n} (0) + \lambda \|0\|_H^2 =0. $$
Hence we have almost surely
$$ \lambda \|f_{L, {\mathbb D}_n, \lambda}\|_H^2 \leq -{\mathcal R}_{L^*, {\mathbb D}_n} (f_{L,
{\mathbb D}_n, \lambda}) \leq \frac{1}{n} \sum_{i=1}^n L(x_i, y_i,
0) \leq |L|_0. $$ Then our desired bound follows.
\end{proof}

Applying Proposition \ref{generalThm} to $R= \sqrt{|L|_0/\lambda}$ gives a learning rate.
But we can do better by an iteration technique.
However, we will first give the proof of Theorem \ref{capacityThm}.

%------------------------------------------------------------------------------
\subsection{Proofs of the main results in Section \ref{RatesSection}}

\proof[Proof of Theorem \ref{capacityThm}] By the definition of
the $\ell_2$-empirical covering number, for every $j\in \{1,
\ldots, s\}$ and ${\bf x}^{(j)} \in ({\mathcal X}_j)^n$, there
exists a set of functions $\{f^{(j)}_i: i=1, \ldots, {\mathcal
N}^{(j)}\}$ with ${\mathcal N}^{(j)} = {\cal N}\left(\{f\in H_j:
\|f\|_{H_j} \leq 1\}, \epsilon\right)$ such that for every
$f^{(j)} \in H_j$ with $\|f^{(j)}\|_{H_j} \leq 1$ we can find some
$i_j \in \{1, \ldots, {\mathcal N}^{(j)}\}$ satisfying $d_{2, {\bf
x}^{(j)}}(f^{(j)}, f^{(j)}_{i_j}) \leq \epsilon$.

Now every function $f\in H$ with $\|f\|_{H} \leq 1$ can be written
as $f= f^{(1)} + \ldots + f^{(s)}$ with $\|f^{(j)}\|_{H_j} \leq
1$. Also, every ${\bf x} =(x_\ell)_{\ell=1}^n \in ({\mathcal
X})^n$ can be expressed as $x_\ell = \left(x^{(1)}_\ell, \ldots,
x^{(s)}_\ell\right)$ with ${\bf x}^{(j)}= (x^{(j)}_\ell)_{\ell
=1}^n \in ({\mathcal X}_j)^n$. By taking the function $f_{i_1,
\ldots, i_s} = f^{(1)}_{i_1} + \ldots + f^{(s)}_{i_s}$, we see
that
\begin{eqnarray*}
d_{2, {\bf x}}\left(f, f_{i_1, \ldots, i_s}\right) &=& \left\{\frac{1}{n}\sum_{\ell=1}^n\big(f(x_\ell)-f_{i_1, \ldots, i_s} (x_\ell)\big)^2\right\}^{1/2} \\
&=& \Big\{\frac{1}{n}\sum_{\ell=1}^n\Big(\left(f^{(1)}(x^{(1)}_\ell) + \ldots + f^{(s)}(x^{(s)}_\ell)\right) \\
&& - \left(f^{(1)}_{i_1}(x^{(1)}_\ell) + \ldots + f^{(s)}_{i_s}(x^{(s)}_\ell)\right)\Big)^2\Big\}^{1/2} \\
&\leq& \sum_{j=1}^s
\Big\{\frac{1}{n}\sum_{\ell=1}^n\Big(f^{(j)}(x^{(j)}_\ell)
- f^{(j)}_{i_j}(x^{(j)}_\ell)\Big)^2\Big\}^{1/2} \\
&=& \sum_{j=1}^s d_{2, {\bf x}^{(j)}} \left(f^{(j)}, f^{(j)}_{i_j}\right) \leq s \epsilon.
\end{eqnarray*}
The number of functions of the form $f_{i_1, \ldots, i_s}$ is $\Pi_{j=1}^s {\mathcal N}^{(j)}$.
Therefore,
\begin{eqnarray*}
\log {\cal N}\left(\{f\in H: \|f\|_{H} \leq 1\}, s \epsilon\right)
& \leq & \sum_{j=1}^s \log {\cal N}\left(\{f\in H_j: \|f\|_{H_j} \leq 1\}, \epsilon\right) \\
& \leq & s c_\zeta \left(\frac{1}{\epsilon}\right)^\zeta.
\end{eqnarray*}
Then our desired statement follows by scaling $R$ to $1$.
\qed

We are now in a position to prove our main results stated in Section \ref{RatesSection}. Theorem
\ref{mainratesThm} is proved by applying Proposition \ref{generalThm}
iteratively. The iteration technique for analyzing regularization schemes has been well developed in the literature
\cite{SteinwartScovel, WuYingZhou, Hu, HuFanWuZhou}.

\proof[Proof of Theorem \ref{mainratesThm}] Take $R^{[0]} =
\max\{\sqrt{|L|_0}, 1\} \frac{1}{\sqrt{\lambda}}$. Lemma
\ref{initialB} tells us that ${\mathcal W}(R^{[0]}) ={\mathcal
Z}^n$. We apply an iteration technique with a sequence of radii
$\{R^{[\ell]} \geq 1\}_{\ell \in \NN}$ to be defined below.

Apply Proposition \ref{generalThm} to $R= R^{[\ell]}$, and when $0\leq \theta <1$, apply the elementary inequality
$${1 \over q} + {1 \over q^*} =1 \ \hbox{with} \ q, q^* >1
\Longrightarrow a \cdot b \le {1 \over q} a^q + {1 \over q^*}
b^{q^*}, \qquad \forall a, b \ge 0 $$ with $q={1 \over \theta},
q^* ={1 \over 1-\theta}$ and
$$a=2^{-\theta} \left({\mathcal R}_{L^*, P} (f_{L, {\mathbb D}_n, \lambda}) -{\mathcal R}_{L^*, P} (f^*_{{\mathcal F}, P})\right)^\theta, \quad b=2^{\theta} C'_2 R^{1-\theta} n^{-\frac{2(1-\theta)}{4-2\theta+ \zeta\theta}}. $$
We know that there exists a subset ${\mathcal V}_{R^{[\ell]}}$ of
${\mathcal Z}^n$ with measure at most $\delta$ such that
\begin{eqnarray*}
&& {\mathcal R}_{L^*, P} (f_{L, {\mathbb D}_n, \lambda}) - {\mathcal
R}^*_{L^*, P, {\mathcal F}} + \lambda \|f_{L, {\mathbb D}_n,
\lambda}\|^2_H  \\
&& \leq \frac{1}{2} \left\{{\mathcal R}_{L^*, P} (f_{L, {\mathbb D}_n, \lambda}) -{\mathcal R}_{L^*, P} (f^*_{{\mathcal F}, P})\right\} \\
&& \quad + \left(2^{\theta} C'_2\right)^{\frac{1}{1-\theta}} R^{[\ell]} n^{-\frac{2}{4-2\theta+ \zeta\theta}}   \\
&& \quad  + C'_1 \log \frac{2}{\delta} \max\left\{\frac{\lambda^{\frac{r-1}{2}}}{n}, \frac{\lambda^{\frac{r-1}{2} + \frac{\theta(r+1)}{4}}}{\sqrt{n}}\right\}  \\
&& \quad  + C''_2 \log \frac{2}{\delta} R^{[\ell]}
n^{-\frac{2}{4-2\theta+ \zeta\theta}} + C_r \lambda^{r}, \qquad
\forall {\bf z} \in {\mathcal W}(R^{[\ell]})\setminus {\mathcal
V}_{R^{[\ell]}}.
\end{eqnarray*}
It follows that when $\lambda = n^{-\beta}$ for some $\beta >0$,
we have
\begin{equation}\label{iterGeneral}
 {\mathcal R}_{L^*, P} (f_{L, {\mathbb D}_n, \lambda}) - {\mathcal
R}^*_{L^*, P, {\mathcal F}} + \lambda \|f_{L, {\mathbb D}_n,
\lambda}\|^2_H \leq \max\left\{a_{n, \delta} R^{[\ell]}, b_{n,
\delta}\right\}, \ \forall {\bf z} \in {\mathcal
W}(R^{[\ell]})\setminus {\mathcal V}_{R^{[\ell]}},
\end{equation} where
$$ a_{n, \delta} := \left\{4 \left(2^{\theta} C'_2\right)^{\frac{1}{1-\theta}} + 4 C''_2\right\} \log \frac{2}{\delta} n^{-\frac{2}{4-2\theta+ \zeta\theta}} $$
and
$$b_{n, \delta} :=\left\{4 C'_1 + 4 C_r\right\} \log \frac{2}{\delta}
n^{-\alpha'}$$ with
$$\alpha' := \min\left\{\frac{1}{2} +
\beta\left(\frac{\theta (1+r)}{4} - \frac{1-r}{2}\right), \
r\beta\right\}. $$ Thus we have
$$\|f_{L, {\mathbb D}_n,
\lambda}\|_H \leq \max\left\{n^{\frac{\beta}{2}} \sqrt{a_{n,
\delta}} \sqrt{R^{[\ell]}}, n^{\frac{\beta}{2}} \sqrt{b_{n,
\delta}}\right\}, \quad \forall {\bf z} \in {\mathcal
W}(R^{[\ell]})\setminus {\mathcal V}_{R^{[\ell]}}. $$ Hence
\begin{equation}\label{iterR}
{\mathcal W}(R^{[\ell]}) \subseteq {\mathcal W}(R^{[\ell +1]})\cup
{\mathcal V}_{R^{[\ell]}},
\end{equation}
after we define the sequence of radii $\{R^{[\ell]} \geq 1\}_{\ell
\in \NN}$ by
\begin{equation}\label{sequenceR}
R^{[\ell +1]} = \max\left\{n^{\frac{\beta}{2}} \sqrt{a_{n,
\delta}} \sqrt{R^{[\ell]}}, n^{\frac{\beta}{2}} \sqrt{b_{n,
\delta}}, 1\right\}.
\end{equation}
For any positive integer $J \in \NN$, we have
$${\mathcal Z}^m = {\mathcal W}(R^{[0]}) \subseteq {\mathcal W}(R^{ })\cup
{\mathcal V}_{R^{[0]}} \subseteq \ldots \subseteq {\mathcal
W}(R^{[J]})\cup \left(\cup_{\ell =0}^{J-1} {\mathcal
V}_{R^{[\ell]}}\right), $$ which tells us that the set ${\mathcal
W}(R^{[J]})$ has measure at least $1- J \delta$. We also see
iteratively from the definition (\ref{sequenceR}) that
\begin{eqnarray*} R^{[J]} &\leq& \max\left\{n^{\frac{\beta}{2}}
\sqrt{a_{n, \delta}} \sqrt{R^{[J-1]}}, \ n^{\frac{\beta}{2}}
\sqrt{b_{n, \delta}}, \ 1\right\}
\leq \ldots \\
&\leq& \max\Bigl\{\left(n^{\frac{\beta}{2}} \sqrt{a_{n,
\delta}}\right)^{1 + \frac{1}{2} + \ldots + \frac{1}{2^{J-1}}}
\left(R^{[0]}\right)^{\frac{1}{2^J}}, \ n^{\frac{\beta}{2}}
\sqrt{b_{n, \delta}}, \ 1, \ \ldots, \ \\
&& \left(n^{\frac{\beta}{2}} \sqrt{a_{n, \delta}}\right)^{1 +
\frac{1}{2} + \ldots + \frac{1}{2^{J-1}}}
\left(\max\left\{n^{\frac{\beta}{2}} \sqrt{b_{n,
\delta}}, 1\right\}\right)^{\frac{1}{2^{J-1}}} \Bigr\}\\
&\leq& \left\{4 \left(2^{\theta} C'_2\right)^{\frac{1}{1-\theta}}
+ 4 C''_2 + 4 C'_1 + 4 C_r +1\right\}\max\{\sqrt{|L|_0}, 1\} \log
\frac{2}{\delta} n^{\alpha''},
\end{eqnarray*}
where
\begin{eqnarray*} \alpha'' &=& \max\Biggl\{\left(2 - \frac{1}{2^{J-1}}\right) \left(\frac{\beta}{2}-\frac{1}{4-2\theta+ \zeta\theta}\right) + \frac{\beta}{2^{J+1}}, \ \frac{\beta}{2} - \frac{\alpha'}{2}, \\
&& \qquad \left(\frac{\beta}{2}-\frac{1}{4-2\theta+
\zeta\theta}\right) \frac{1}{2} \left(\frac{\beta}{2} -
\frac{\alpha'}{2}\right), \ \ldots, \ \\
&& \qquad \left(2 - \frac{1}{2^{J-1}}\right)
\left(\frac{\beta}{2}-\frac{1}{4-2\theta+ \zeta\theta}\right) +
\frac{1}{2^{J-1}} \left(\frac{\beta}{2} -
\frac{\alpha'}{2}\right)\Biggr\} \\
&\leq& \max\left\{\beta -\frac{2}{4-2\theta+ \zeta\theta}, \
\frac{\beta}{2} - \frac{\alpha'}{2}\right\} +\frac{1}{2^{J}} \\
&=& \max\left\{\beta -\frac{2}{4-2\theta+ \zeta\theta}, \
\frac{(1-r) \beta}{2}, \  \frac{(1-r) \beta}{2} +
\frac{\beta(1+r)(1- \frac{\theta}{2}) -1}{4}\right\}
+\frac{1}{2^{J}}.
\end{eqnarray*}

Denote
$$ \alpha''' =\max\left\{\beta -\frac{2}{4-2\theta+ \zeta\theta}, \
\frac{(1-r) \beta}{2}, \  \frac{(1-r) \beta}{2} +
\frac{\beta(1+r)(1- \frac{\theta}{2}) -1}{4}\right\} $$ and the
constant
$$ C_3 =\left\{4 \left(2^{\theta} C'_2\right)^{\frac{1}{1-\theta}}
+ 4 C''_2 + 4 C'_1 + 4 C_r\right\}\max\{\sqrt{|L|_0}, 1\}. $$
Choose $J$ to be the smallest positive integer greater than or
equal to $\log_2 \frac{1}{\epsilon}$. Then $\frac{1}{2^{J}} \leq
\epsilon$ and
$$R^{[J]} \leq C_3 \log
\frac{2}{\delta} n^{\alpha''' + \epsilon}.
$$
Applying (\ref{iterGeneral}) to $\ell =J$, we know that for every
${\bf z} \in {\mathcal W}(R^{[J]})\setminus {\mathcal
V}_{R^{[J]}}$, there holds
\begin{eqnarray*} && {\mathcal R}_{L^*, P}
(f_{L, {\mathbb D}_n, \lambda}) - {\mathcal R}^*_{L^*, P,
{\mathcal F}} \leq \max\left\{a_{n, \delta} R^{[J]}, b_{n,
\delta}\right\} \\
&&\leq \left(C_3 \left\{4 \left(2^{\theta}
C'_2\right)^{\frac{1}{1-\theta}} + 4 C''_2\right\} + \left\{4 C'_1
+ 4 C_r\right\}\right) \left(\log \frac{2}{\delta}\right)^2
n^{\epsilon-\min\left\{\frac{2}{4-2\theta+ \zeta\theta} -
\alpha''', \ \alpha'\right\}}. \end{eqnarray*} Since the set
${\mathcal W}(R^{[J]})$ has measure at least $1-J \delta$ while
the set ${\mathcal V}_{R^{[J]}}$ has measure at most $\delta$, we
know that with confidence at least $1- (J +1) \delta$,
$${\mathcal R}_{L^*, P}
(f_{L, {\mathbb D}_n, \lambda}) - {\mathcal R}^*_{L^*, P,
{\mathcal F}} \leq \widetilde{C} \left(\log
\frac{2}{\delta}\right)^2 m^{\epsilon- \alpha}, $$ where
$$ \alpha = \min\left\{\frac{2}{4-2\theta+ \zeta\theta} -
\alpha''', \ \alpha'\right\}$$ and $$ \widetilde{C}= \left(C_3
\left\{4 \left(2^{\theta} C'_2\right)^{\frac{1}{1-\theta}} + 4
C''_2\right\} + \left\{4 C'_1 + 4 C_r\right\}\right). $$ Scaling
$J \delta$ to $\delta$, and expressing $\alpha$ explicitly, we see
that the conclusion of Theorem \ref{mainratesThm} holds true. \qed

It only remains to prove Theorem \ref{quantileThm}. We will do so by showing that Theorem \ref{quantileThm} is a special case
of Theorem \ref{mainratesThm}.

\proof[Proof of Theorem \ref{quantileThm}]  Since $P$ has a $\tau$-quantile of $p$-average type $2$ for some $p\in (0, \infty]$, we know from
\cite{SC1} that Assumption \ref{assumption3} holds true with $\theta = \frac{p}{p+1}$. Since ${\mathcal X}_j \subset \RR^{d_j}$ and $k_j \in C^\infty ({\mathcal X}_j \times {\mathcal X}_j)$, we know from
\cite{Zhoucap} that
Assumption \ref{assumption2} holds true for an arbitrarily small $\zeta>0$. By inserting $r=\frac{1}{2}$, $\beta = \frac{4(p+1)}{3(p+2)}$ and $\theta = \frac{p}{p+1}$ into the expression of $\alpha$ in Theorem \ref{mainratesThm} and choosing $\zeta$ to be sufficiently small, we know that the conclusion of Theorem \ref{quantileThm} follows from that of Theorem \ref{mainratesThm}.
\qed

\vspace{1cm}
\noindent \textbf{Addresses:}\\[1em]
\begin{tabular}{p{8cm}p{8cm}}
Andreas Christmann & Ding-Xuan Zhou \\
University of Bayreuth & City University of Hong Kong\\
Department of Mathematics & Department of Mathematics\\
Universitaetsstr. 30 &  Y6524 (Yellow Zone), 6/F Academic 1\\
D-95447 Bayreuth   & Tat Chee Avenue\\
Germany   & Kowloon Tong \\
 & Hong Kong\\
 & China
\end{tabular}


\addcontentsline{toc}{Chapter}{Bibliography}
\begin{thebibliography}{99}

\bibitem{Bach2008} F. Bach, Consistency of the Group Lasso and Multiple Kernel Learning,
  {\it J. Mach. Learn. Res.}, {\bf 9} (2008), 1179--1225.

\bibitem{BoserGuyonVapnik1992} B.E. Boser, I. Guyon, and V. Vapnik, A training algorithm for optimal margin classifiers,
 in: Proceedings of the Fifth Annual ACM Workshop on Computational Learning Theory, pp. 144--152, ACM, Madison, WI, 1992.

\bibitem{ChristmannHable2012} A. Christmann and R. Hable,
Consistency of support vector machines using additive kernels for additive models,
{\em Computational Statistics and Data Analysis} {\bf 56} (2012), 854--873.

\bibitem{ChristmannVanMessemSteinwart2009}
   A. Christmann, A. {V}an Messem, and I. Steinwart,
   On consistency and robustness properties of support vector machines for heavy-tailed distributions,
   {\em Statistics and Its Interface} {\bf 2} (2009), 311--327.

\bibitem{CortesVapnik1995} C. Cortes and V. Vapnik, Support vector networks, {\it Mach. Learn.} {\bf 20} (1995), 273--297.

\bibitem{CuckerZhou2007} F. Cucker and D. X. Zhou, Learning Theory. An Approximation Theory Viewpoint, Cambridge University Press, Cambridge, 2007.

\bibitem{EbertsSteinwart2013} M. Eberts and I. Steinwart, Optimal regression rates for SVMs using Gaussian kernels, {\it Electronic Journal of Statistics} {\bf 7} (2013), 1--42.

\bibitem{EdmundsTriebel} D. Edmunds and H. Triebel, Function Spaces, Entropy Numbers, Differential Opretaors, Cambridge University Press, Cambridge, 1996.

\bibitem{HastieTibshirani1986} T. Hastie and R. Tibshirani, Generalized additive models, {\it Statistical Science} {\bf 1} (1986), 297--318.

\bibitem{HastieTibshirani} T. J. Hastie and R. J. Tibshirani, Generalized Additive Models, CRC Press, 1990.

\bibitem{HofmannSchoelkopfSmola2008} T. Hofmann, B. Sch{\"o}lkopf, and A. J. Smola, Kernel methods in machine learning, {\it Ann. Statist.} {\bf 36} (2008), 1171--1220.

\bibitem{Hu} T. Hu, Online regression with varying Gaussians and non-identical distributions, {\it Analysis and Applications} {\bf 9} (2011), 395--408.

\bibitem{HuFanWuZhou} T. Hu, J. Fan, Q. Wu, and D. X. Zhou, Regularization schemes for minimum error entropy principle, {\it Analysis and Applications}, to appear.

\bibitem{Huber1967} P. J. Huber, The behavior of maximum likelihood estimates under nonstandard conditions,
   {\em Proc. 5th Berkeley Symp.} {\bf 1} (1967), 221--233.
   
\bibitem{Koenker2005} R. Koenker, Quantile Regression, Cambridge University Press, Cambridge.

\bibitem{KoenkerBassett1978} R. Koenker and G. Bassett, Regression Quantiles, {\it Econometrica} {\bf 46}, 1978, 33--50.
   
\bibitem{KoltchinskiiYuan2008} V. Koltchinskii and M. Yuan, Sparse Recovery in Large Ensembles of Kernel Machines,
  In: {\it Proceedings of COLT}, (2008).
  
\bibitem{LinZhang2006} Y. Lin and H.H. Zhang, Component Selection and Smoothing in Multivariate Nonparametric Regression,
  {\it Ann. Statist.} {\bf 34} (2006), 2272-2297.
  
\bibitem{MeierVandeGeerBuehlmann2009} L. Meier, S. van de Geer, and P. B{\"u}hlmann, High-dimensional Additive Modeling,
  {\it Ann. Statist.} {\bf 37} (2009), 3779--3821.

\bibitem{Minh} H. Q. Minh, Some properties of {G}aussian reproducing
kernel {H}ilbert spaces and their implications for function
approximation and learning theory, {\em Constr. Approx.} {\bf 32} (2010), 307--338.

\bibitem{PoggioGirosi1990} T. Poggio and F. Girosi, A theory of networks for approximation and learning, {\it Proc. IEEE} {\bf 78} (1990), 1481--1497.

\bibitem{RaskuttiWainwrightYu2012} G. Raskutti, M.J. Wainwright, and B. Yu, 
  Minimax-Optimal Rates for Sparse Additive Models Over Kernel Classes Via Convex Programming,
  {\it J. Mach. Learn. Res.}  {\bf 13} (2012), 389-427.

\bibitem{SchoelkopfSmola2002} B. Sch{\"o}lkopf and A. J. Smola, Learning with Kernels, MIT Press, Cambridge, M.A., 2002.

\bibitem{SchoelkopfEtAl2000} B. Sch{\"o}lkopf, A.J. Smola, R.C. Williamson, and P.L. Bartlett,
  New support vector algorithms, {\it Neural Comput.} {\bf 12} (2000), 1207--1245.

\bibitem{SZII} S. Smale and D. X. Zhou, Shannon sampling II.
Connections to learning theory, {\it Appl. Comput. Harmonic Anal.} {\bf 19} (2005), 285--302.

\bibitem{SteinwartChristmann2008a} I. Steinwart and A. Christmann, {S}upport {V}ector {M}achines, Springer, New York, 2008.

\bibitem{SC1} I. Steinwart and A. Christmann, How {SVMs} can estimate quantiles and the median, 
  {\em Advances in Neural Information Processing Systems} {\bf 20} (2008), 305-312, MIT Press, Cambridge, MA.

\bibitem{SC2} I. Steinwart and A. Christmann, Estimating conditional quantiles with
the help of the pinball loss, {\em Bernoulli} {\bf 17} (2011), 211--225.

\bibitem{SteinwartScovel} I. Steinwart and C. Scovel, Fast rates for support vector machines using {G}aussian kernels, {\it Ann. Statist.} {\bf 35} (2007), 575--607.

\bibitem{Stone1985} C. J. Stone, Additive regression and other nonparametric models, {\it Ann. Statist.} {\bf 13} (1985), 689--705.

\bibitem{SunWu} H. Sun and Q. Wu, Indefinite kernel network with dependent sampling, {\it Anal. Appl.} {\bf 11} (2013), 1350020, 15 pages.

\bibitem{SuykensEtAl2002} J.~A.~K. Suykens, T. {Van~Gestel}, J. {De Brabanter}, B. {De~Moor}, and J. Vandewalle,
   Least Squares Support Vector Machines, World Scientific, Singapore, 2002.

\bibitem{SuzukiSugiyama2013} T. Suzuki and M. Sugiyama, Fast learning rate of multiple kernel learning: trade-off between 
  sparsity and smoothness, {\it Ann. Statist.} {\bf 41} (2013), 1381--1405.

\bibitem{TakeuchiEtAl2006} I. Takeuchi, Q.V. Le, T.D. Sears, and A.J. Smola, Nonparametric quantile estimation,
  {\it J. Mach. Learn. Res.} {\bf 7} (2006), 1231--1264.
   
\bibitem{VapnikLerner1963} V. N. Vapnik and A. Lerner, Pattern recognition using generalized portrait method, {\it Autom. Remote Control.} {\bf 24} (1963), 774--780.

\bibitem{Vapnik1995} V. N. Vapnik, The Nature of Statistical Learning Theory, Springer, New York, 1995.

\bibitem{Vapnik1998} V. N. Vapnik, Statistical Learning Theory, John Wiley \& Sons, New York, 1998.

\bibitem{Wahba1999} G. Wahba, Support vector machines, reproducing kernel Hilbert spaces and the randomized GACV,
   in: Advances in Kernel Methods -- Support Vector Learning (Eds. B. Sch{\"o}lkopf, C.~J.~C. Burges, and A.J. Smola),
   MIT Press, Cambridge, MA, (1999), 69--88.

\bibitem{Wendland2005} H. Wendland, Scattered Data Approximation, Cambridge University Press, Cambridge, 2005.

\bibitem{WuYingZhouFoCM} Q. Wu, Y. M. Ying, and D. X. Zhou, Learning rates of least square regularized regression, {\it Found. Comput. Math.} {\bf 6} (2006), 171--192.

\bibitem{WuYingZhou} Q. Wu, Y. M. Ying, and D.-X. Zhou, Multi-kernel regularized classifiers, {\it J. Complexity} {\bf 23} (2007), 108--134.

\bibitem{Xiang} D. H. Xiang, Conditional quantiles with varying Gaussians, {\it Adv. Comput. Math.} {\bf 38} (2013), 723--735.

\bibitem{XiangZhou2009} D. H. Xiang and D.-X. Zhou, Classification with Gaussians and Convex Loss, {\it J. Mach. Learn. Res.} {\bf 10} (2009), 1447--1468.

\bibitem{Zhoucap} D. X. Zhou, Capacity of reproducing kernel spaces in learning theory, {\it IEEE Trans. Inform. Theory} {\bf 49} (2003), 1743-1752.

\end{thebibliography}
\end{document}